\documentclass{article} 
\usepackage{iclr2027_conference,times}

\usepackage{hyperref}
\usepackage{url}
\usepackage{amsmath,amssymb,amsthm}
\usepackage{booktabs}
\usepackage{graphicx}
\usepackage{multirow}
\usepackage{enumitem}
\usepackage[table]{xcolor}
\usepackage[normalem]{ulem}   

\graphicspath{{fig/}}

\theoremstyle{plain}
\newtheorem{proposition}{Proposition}
\newtheorem{corollary}{Corollary}
\theoremstyle{definition}
\newtheorem{assumption}{Assumption}
\newtheorem{definition}{Definition}
\newtheorem{construction}{Construction}
\theoremstyle{remark}
\newtheorem{remark}{Remark}

\newcommand{\Jobj}{J}
\newcommand{\Frag}{\Phi}
\newcommand{\Bnbhd}{\mathcal{B}}
\newcommand{\Bank}{S}
\newcommand{\Bankstar}{S^{*}}
\newcommand{\pol}{\pi_{\theta}}
\newcommand{\Exp}{\mathbb{E}}

\newcommand{\normcdf}{\Phi_{\mathcal{N}}}
\newcommand{\CREPE}{CODA}  

\newcommand{\SR}{\mathrm{SR}}
\newcommand{\CR}{\mathrm{CR}}   

\newcommand{\RR}{\mathrm{RR}}
\newcommand{\Banknone}{\varnothing}
\newcommand{\netben}{\Delta_{\mathrm{net}}}
\newcommand{\Drop}{\mathrm{Drop}}
\newcommand{\RP}{\mathrm{RP}}  

\definecolor{diffmark}{HTML}{0B6E6E}
\definecolor{codahighlight}{HTML}{E6E3F5}                 
\newcommand{\chg}[1]{\textcolor{diffmark}{\textbf{#1}}}   
\newcommand{\del}[1]{\textcolor{diffmark}{\sout{#1}}}     

\title{Evolving Together, Failing Apart:\\ The Skill--Policy Co-evolution Trap in Self-Improving Agents}

\author{Yujie Feng$^{1}$,\enspace Xiaobin Wang$^{1}$,\enspace Jinluan Yang$^{1}$,\enspace Xiaonan Shi$^{1}$ \\
\textbf{Pengjun Xie$^{1}$,}\enspace \textbf{Philip S. Yu$^{3}$,}\enspace \textbf{Xiaoming Wu$^{2}$,}\enspace \textbf{Jieping Ye$^{1}$}
\\
$^{1}$Alibaba Token Hub, Alibaba Group \quad
$^{2}$The Hong Kong Polytechnic University \\
$^{3}$University of Illinois at Chicago
}

\iclrfinalcopy 

\begin{document}

\maketitle
\lhead{Preprint. Under review.}

\begin{abstract}
Skill--policy co-evolution---in which an agent's policy model and its skill bank improve each other
during training---is a promising paradigm for self-improving agents. 
Current evaluations score the co-evolved policy on the final bank it trained with (the \emph{delivered bank}).
After deployment, however, skills may be rewritten, added or removed more frequently than the policy model is retrained, leaving the policy to operate on a drifted skill bank that neither the training objective nor existing evaluation protocols address.
In this paper, we first establish, through theoretical and empirical analysis, that the fragility hidden by this blind spot is not incidental.
The acceptance rule these methods share (\emph{cooperative acceptance}) ties a policy's competence to the bank it co-evolved with, leaving it fragile once skills drift, regardless of the implementation; 
empirically, sensitivity to skill drift grows as co-evolution training proceeds.
We therefore introduce the \textbf{CO}-evolution \textbf{D}rift \textbf{A}udit (\textbf{\CREPE{}}), a
deployment-oriented evaluation protocol that simulates skill drift and, alongside success rate under worst-case drift, introduces two new metrics: a \emph{consistency rate} under meaning-preserving edits and a \emph{resistance rate} to corrupted skills.
Across four co-evolution methods and five benchmarks (six settings), worst-case success is \textbf{5--13} points below the peak measured for each delivered pair, erasing \textbf{52--93\%} of the delivered bank's gain over the no-skill baseline. \CREPE{} thus reframes the target of co-evolution---from peak performance on the co-evolved bank to the ability to harness skills however they evolve---and provides a shared standard against which future methods can be audited.
\end{abstract}

\section{Introduction}
\label{sec:intro}

Building agents that improve themselves is a central goal of current agent research, and
co-evolving a skill library with a policy is one of the most promising routes~\citep{xia2026skillrl,he2026reskill}. The agent
solves tasks, distills the experience into reusable skills, and uses those skills to
strengthen its policy in turn, so competence compounds over rounds. A recent line of
methods~\citep{shi2026skill1,li2026arise,zhang2026co,yang2026skillmaster, lin2026auso}
follows exactly this recipe: it iterates the policy and the skill bank together during
training and reports large gains on agent benchmarks. 
Each ultimately delivers a pair---a co-evolved policy and its final skill bank (the \emph{delivered bank})---and reports the policy's task success on that bank.

However, current evaluation is disconnected from deployment: it assumes that the skill bank remains fixed after training. Real skill banks are non-stationary: skills are rewritten, entries are added or removed, and versions coexist, often changing faster than the policy is retrained. A deployed policy therefore typically runs on a bank that has drifted from the one it was delivered with, yet existing protocols do not test whether performance is maintained. 
A few works note this brittleness~\citep{liu2026rethinking,hebbar2026sia} but offer neither a systematic measurement nor a causal account.


We argue that the blind spot conceals a property of the training objective, not an incidental evaluation artifact. Co-evolution repeatedly adapts the policy and bank into a mutually specialized pair, while evaluation rewards exactly that pairing. The policy is therefore optimized for its delivered bank rather than for transferable skill use. We establish this in two ways. Theoretically, we prove that the shared acceptance rule---retaining a skill update only when it helps the current policy---places the delivered pair at a peak over candidate banks, so rewriting the delivered skills cannot improve the policy and usually degrades its performance (Section~\ref{sec:theory}). Empirically, using the proposed \emph{co-evolution fragility curve}, we show that sensitivity to skill drift increases as training proceeds in five of six settings (Section~\ref{sec:empirics}). We call this the \emph{co-evolution trap} (Figure~\ref{fig:teaser}): the process that raises peak performance also creates sensitivity to bank drift.

\begin{figure}[t]
\begin{center}
\includegraphics[width=\textwidth]{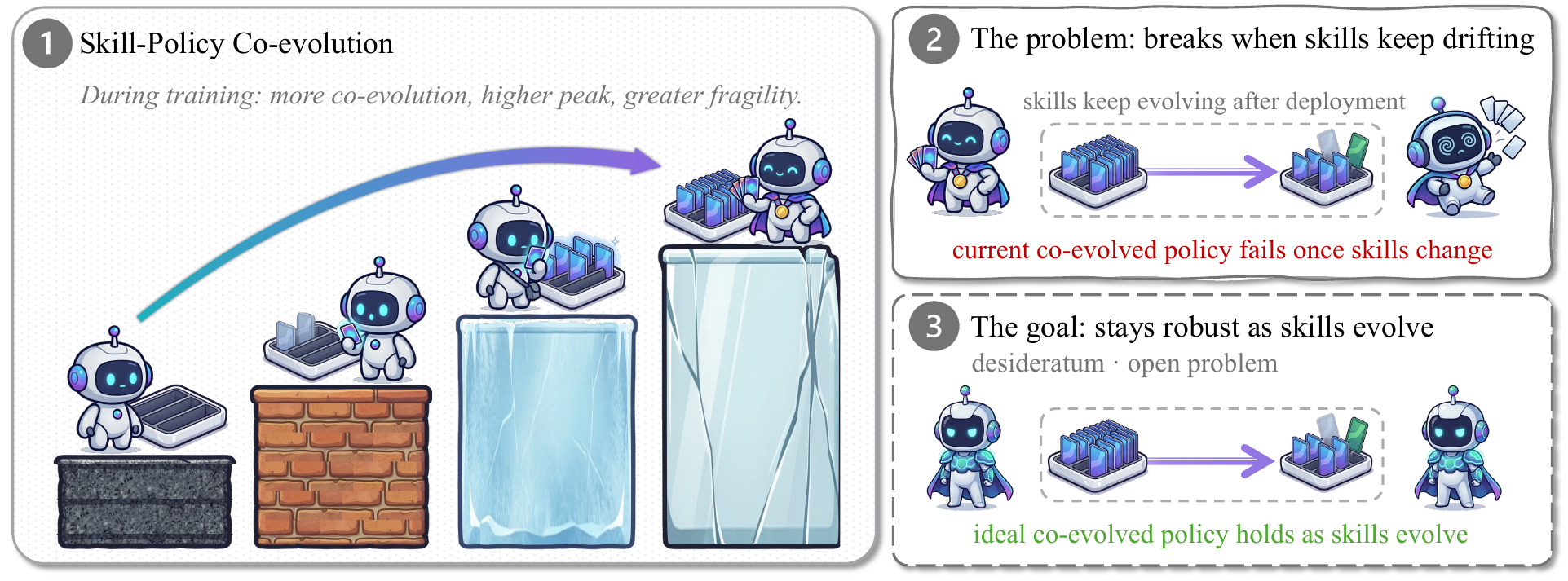}
\end{center}
\caption{\textbf{The skill--policy co-evolution trap (conceptual illustration)}. 
\textbf{(1)~Training}: 
The policy and skill bank co-evolve round by round,
and evaluation reports the peak reached by the resulting delivered pair.
\textbf{(2)~Deployment}: The skill bank continues to evolve while the policy remains fixed. Because the policy is specialized to its delivered bank, its performance degrades as the bank drifts.
\textbf{(3)~Goal}: An ideally co-evolved policy that retains its performance as skills
evolve.
}
\label{fig:teaser}
\end{figure}

Establishing the problem exposes a second gap: existing evaluations cannot measure it,
because they report only one success rate on the fixed delivered bank. We therefore
introduce the \textbf{CO}-evolution \textbf{D}rift \textbf{A}udit (\textbf{\CREPE{}}), a deployment-oriented protocol for
auditing post-deployment skill drift under controlled conditions. \CREPE{} supplements the
conventional peak success rate with worst-case success and two new metrics---a
\emph{consistency rate} under meaning-preserving edits and a \emph{resistance rate} to
corrupted skills. Across four methods and five benchmarks (six settings), worst-case success
is \textbf{5--13 points} below the peak measured for each delivered pair, erasing \textbf{52--93\%} of the delivered bank's gain over the no-skill baseline, so the reported peak overstates the competence a deployment can rely on. \CREPE{} thus reframes the target of
co-evolution---from achieving peak performance on the bank a method co-evolved with to using skills reliably as
they keep changing---and we release it as a common standard for auditing future methods.

Our main contributions are summarized as follows:

\begin{itemize}[leftmargin=*]
\item We \textbf{expose a trap in skill--policy co-evolution}: current evaluations measure fit to a fixed delivered bank, leaving performance under deployment-time bank change unmeasured.


\item We \textbf{theoretically and empirically demonstrate} its mechanism and signature:
under cooperative acceptance and candidate coverage, no covered rewrite improves the
delivered policy; empirically, 18/18 endpoint pairs obey this prediction and fragility grows along training in five of six settings.


\item We \textbf{propose a deployment-oriented evaluation protocol}, \CREPE{}, reporting worst-case success under tested content-level drift, consistency under meaning-preserving edits, and resistance to corrupted skills, making the deployment robustness of the delivered policy--bank pair measurable.

\end{itemize}

\section{Why Co-evolution Favors Specialization over Robustness}
\label{sec:theory}

\paragraph{Problem setting.}
\label{ssec:setting}
A task $x$ is drawn from a distribution $D$, a skill bank $\Bank$ lies in a space
$\mathcal{S}$, and a policy $\pol(\cdot \mid x, \Bank)$ maps both to a distribution over
action sequences, with skills entering as text in its context. Methods optimize and report
the expected return
\begin{equation}
\Jobj(\theta, \Bank) \;=\; \Exp_{x \sim D}\!\left[R\!\left(\pol(\cdot \mid x, \Bank)\right)\right],
\label{eq:objective}
\end{equation}
where $R$ is a binary success indicator, so $\Jobj$ is the reported success rate. We fold
skill selection into $\pol$, making $\Jobj$ a function of the bank as a whole.
Co-evolution alternates a \emph{policy step}, which improves $\theta$ against the current bank, and a \emph{bank step}, which proposes banks by rewriting, adding or removing skills and retains one through an acceptance test. 
After $T$ rounds, it delivers $(\theta^{*},\Bankstar)$ and reports $\Jobj(\theta^{*},\Bankstar)$ as its \emph{peak return}---the de facto objective for comparing these methods.

\subsection{A Unified View and the Premise It Reveals}
\label{ssec:axes}

Current methods differ in \emph{update cadence}, \emph{bank growth}, \emph{validation
gate} and \emph{injection mechanism} (Table~\ref{tab:framework}; definitions in
Appendix~\ref{app:framework}). Across these choices, they share two patterns. First, they
optimize and report peak return on the bank they co-evolved with. Second, their gates favor
banks that serve the current policy, whether through reward thresholds, sampled-return
comparisons or posteriors over competing versions. The first motivates the point-mass
diagnosis of Section~\ref{ssec:prop3}; the second motivates the idealized selection rule
below.

No method searches all of $\mathcal{S}$. Let $\mathcal{C}_t\subseteq\mathcal{S}$ denote
the candidate banks available at round $t$.

\begin{table}[t]
\caption{Current skill--policy co-evolution methods differ in implementation but share a
selection bias toward current-policy utility and report peak return on the co-evolved bank.}
\label{tab:framework}
\begin{center}
\scalebox{0.7}{
\begin{tabular}{@{}lllll|cc@{}}
\toprule
 & \multicolumn{4}{c}{\emph{Design choices} (vary across methods)} & \multicolumn{2}{c}{\textbf{Shared structure}} \\
\cmidrule(lr){2-5} \cmidrule(l){6-7}
Method & Cadence & Growth & Gate & Injection & Selection bias & Objective \\
\midrule
SkillRL~\citep{xia2026skillrl}       & alternating & monotone        & success-gated       & batch+top-$k$ & \multirow{8}{*}{\textbf{current utility}} & \multirow{8}{*}{\textbf{peak return}} \\
ReSkill~\citep{he2026reskill}       & near-sync   & capped ($\le$8) & Thompson samp.      & per-step      & & \\
SAPO~\citep{zhang2026co}   & near-sync   & two-tier, capped & marginal util.~$>$0 & rerank       & & \\
Skill1~\citep{shi2026skill1}       & joint       & monotone        & success ($r{=}1$)   & log-prob      & & \\
ARISE~\citep{li2026arise} & joint       & two-tier        & hier.\ reward       & log-prob      & & \\
SkillMaster~\citep{yang2026skillmaster} & near-sync & managed & policy scoring & policy-managed & & \\
Evolving-RL~\citep{fan2026evolving}            & joint       & monotone        & reward thresh.      & batch         & & \\
COSPLAY~\citep{wu2026co}                & near-sync   & managed (lifecycle) & contract verif. & per-step      & & \\
\bottomrule
\end{tabular}}
\end{center}
\end{table}

\begin{assumption}[Idealized cooperative acceptance]
\label{as:accept}
The retained bank maximizes the current policy's return over the round's candidate set,
\begin{equation}
\Bank_t \;\in\; \arg\max_{\Bank \in \mathcal{C}_t} \Jobj(\theta_t, \Bank).
\label{eq:accept}
\end{equation}
\end{assumption}
The rule is \emph{cooperative} because selection serves the current policy rather than
challenging it. Actual gates admit candidates that pass a non-decreasing utility test and
terminate after a finite budget; they need not solve Eq.~\eqref{eq:accept} exactly. We use
the equation as an explicit idealization of their shared selection bias, restricted to the
round's candidate set rather than all of $\mathcal{S}$.

\subsection{Perturbation Response and the Co-evolution Peak}
\label{ssec:fragility}
\label{ssec:prop1}

\begin{definition}[Content-level neighborhood]
\label{def:nbhd}
For a bank $\Bank$, let $\Bnbhd(\Bank)$ denote a specified family of banks obtained by
rewriting the textual content of skills already present while fixing the skill set and
retrieval-side fields. We include $\Bank$ itself and exclude the empty bank.
\end{definition}

\begin{assumption}[Candidate coverage at delivery]
\label{as:coverage}
The final candidate family covers the content-level neighborhood of the delivered bank,
\begin{equation}
\Bnbhd(\Bankstar) \;\subseteq\; \mathcal{C}_T.
\label{eq:nbhd-in-candidates}
\end{equation}
\end{assumption}

Assumption~\ref{as:coverage} makes explicit that the theory applies only to the covered
rewrite family, not to arbitrary post-deployment edits. The empty bank is excluded because
removal is an autonomy anchor rather than a content perturbation
(Section~\ref{ssec:collapse}). Fixing retrieval-side fields holds the selected skill set
constant, as enforced experimentally (Section~\ref{ssec:probe};
Appendix~\ref{app:masking}).

\begin{definition}[Perturbation response and fragility]
\label{def:frag}
The \emph{response set} of a policy at a bank collects the return changes induced by
content-level perturbation,
$\mathcal{R}(\theta, \Bank) = \big\{\Jobj(\theta, \Bank') - \Jobj(\theta, \Bank) : \Bank' \in \Bnbhd(\Bank)\big\}$,
and \emph{fragility} is its worst case in magnitude,
\begin{equation}
\Frag(\theta, \Bank) \;=\; -\min \mathcal{R}(\theta, \Bank) \;=\; \Jobj(\theta, \Bank) - \min_{\Bank' \in \Bnbhd(\Bank)} \Jobj(\theta, \Bank').
\label{eq:frag}
\end{equation}
\end{definition}

The content-level drops measured in Section~\ref{sec:empirics} (axes T1--T3) and the
peak-to-worst gap reported in Section~\ref{ssec:metrics} are empirical lower bounds on
$\Frag$: both evaluate finitely many variants from $\Bnbhd(\Bank)$, so they can only
understate its worst-case value, making our estimates conservative. Since $\Bank \in \Bnbhd(\Bank)$, $\Frag \ge 0$ for \emph{every}
policy, so magnitude alone cannot distinguish a co-evolved policy from any other; what
does is the \emph{sign structure} of $\mathcal{R}$, which Proposition~\ref{prop:peak}
makes precise.

\begin{figure}[t]
\begin{center}
\includegraphics[width=0.85\textwidth]{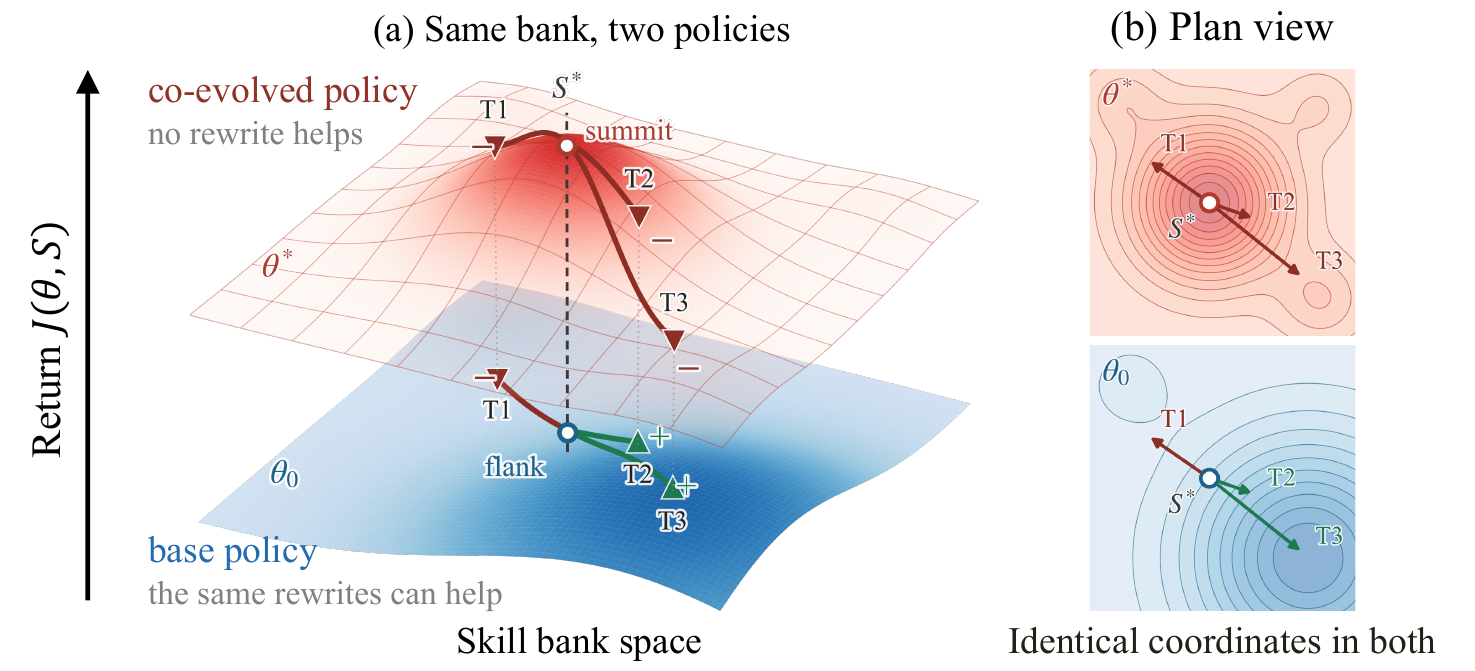}
\end{center}
\caption{Proposition~\ref{prop:peak} as a return landscape (schematic). For the co-evolved
policy $\theta^{*}$, the delivered bank is a peak over the covered family, so no covered
content perturbation ascends. A base policy $\theta_0$ need not share that peak, and the same
perturbations may improve its return. 
}
\label{fig:landscape}
\end{figure}

\begin{proposition}[Co-evolution peak]
\label{prop:peak}
Under Assumptions~\ref{as:accept} and~\ref{as:coverage}, the delivered pair
$(\theta^{*},\Bankstar)=(\theta_T,\Bank_T)$ satisfies
\begin{equation}
\mathcal{R}(\theta^{*}, \Bankstar) \subseteq (-\infty, 0].
\label{eq:nonpositive}
\end{equation}
Thus no covered content-level variant improves the co-evolved policy, and
$\Frag(\theta^{*},\Bankstar)=0$ only if return is constant on
$\Bnbhd(\Bankstar)$. For a policy $\theta_0$ not selected against $\Bankstar$ under these
assumptions, positive responses remain admissible:
$\max\mathcal{R}(\theta_0,\Bankstar)>0$ may hold.
\end{proposition}

Intuitively, selection places the delivered bank at a peak for the policy it co-evolved with,
so no covered direction leads uphill (Figure~\ref{fig:landscape}). The result is one-sided:
a policy not selected against this bank need not lie at the same peak, so a rewrite may help.
Because both assumptions idealize the training process and its coverage, we test their
prediction rather than claim to verify the premises: at delivery, no tested content-level
perturbation improves performance beyond evaluation noise
(Section~\ref{sec:empirics}).

\subsection{The Cause: an Objective That Never Marginalizes over Banks}
\label{ssec:prop3}

\begin{proposition}[The point-mass objective prices degradation at zero]
\label{prop:marginal}
For any policy $\theta$ and distribution $P$ with
$\mathrm{supp}(P)\subseteq\Bnbhd(\Bankstar)$,
\begin{equation}
\Exp_{\Bank \sim P}\!\left[\Jobj(\theta, \Bank)\right] \;=\; \underbrace{\Jobj(\theta, \Bankstar)}_{\text{peak}} \;-\; \underbrace{\Exp_{\Bank \sim P}\!\left[\Jobj(\theta, \Bankstar) - \Jobj(\theta, \Bank)\right]}_{\text{average degradation } \;\le\; \Frag(\theta, \Bankstar)}.
\label{eq:decomposition}
\end{equation}
For the delivered policy $\theta^{*}$ under Assumptions~\ref{as:accept}
and~\ref{as:coverage}, the second term lies in
$[0,\Frag(\theta^{*},\Bankstar)]$. The current point-mass case
$P=\delta_{\Bankstar}$ removes it by construction.
\end{proposition}

Equation~\eqref{eq:decomposition} is an algebraic identity and requires no hypothesis on the
policy class. Under a richness condition, Appendix~\ref{app:proof-marginal} shows
that equal peak return can coexist with different fragility, so optimizing the peak alone
need not identify a robust policy. The issue is therefore not that co-evolution optimizes in
the wrong direction, but that it optimizes at one bank while deployment presents a
neighborhood. How to choose and sample from $P$ without sacrificing peak return remains open.

\paragraph{Scope.}
Neither proposition assumes a particular model class, task distribution or reward; both
apply to any method whose delivered pair satisfies Assumptions~\ref{as:accept}
and~\ref{as:coverage} (proofs in Appendix~\ref{app:proofs}). Under analogous coverage,
reversing the selection preference drives $\Frag$ down instead
(Appendix~\ref{app:adversarial}), so the selection bias---not co-evolution alone---governs
whether robustness emerges. 
The magnitude and training-time evolution of $\Frag$ are examined empirically in Section~\ref{ssec:dose}.

\section{Empirical Evidence: Fragility Grows with Co-evolution}
\label{sec:empirics}

\subsection{A Five-Axis Skill-Drift Probe}
\label{ssec:probe}

To measure the response set of Definition~\ref{def:frag}, we perturb the delivered skill bank while holding the task and policy fixed. Five axes represent increasingly structural
forms of drift. \textbf{T1 (semantic)} rewrites a skill synonymously, changing wording but
preserving meaning. \textbf{T2 (sufficiency)} deletes necessary steps or injects irrelevant
content, leaving the skill on topic but incomplete. \textbf{T3 (misleading)} replaces the
prescribed strategy with an incorrect one. \textbf{T4 (ecosystem)} changes bank composition
by adding, removing or mixing skill versions. \textbf{T5 (autonomy)} removes the bank
entirely. T1--T3 alter existing skill content and lie within
$\Bnbhd(\Bankstar)$; T4 and T5 are reported as structural diagnostics outside it.
Figure~\ref{fig:probe} illustrates all five axes; Appendix~\ref{app:examples} provides detailed examples.

\begin{figure}[t]
\begin{center}
\includegraphics[width=0.9\textwidth]{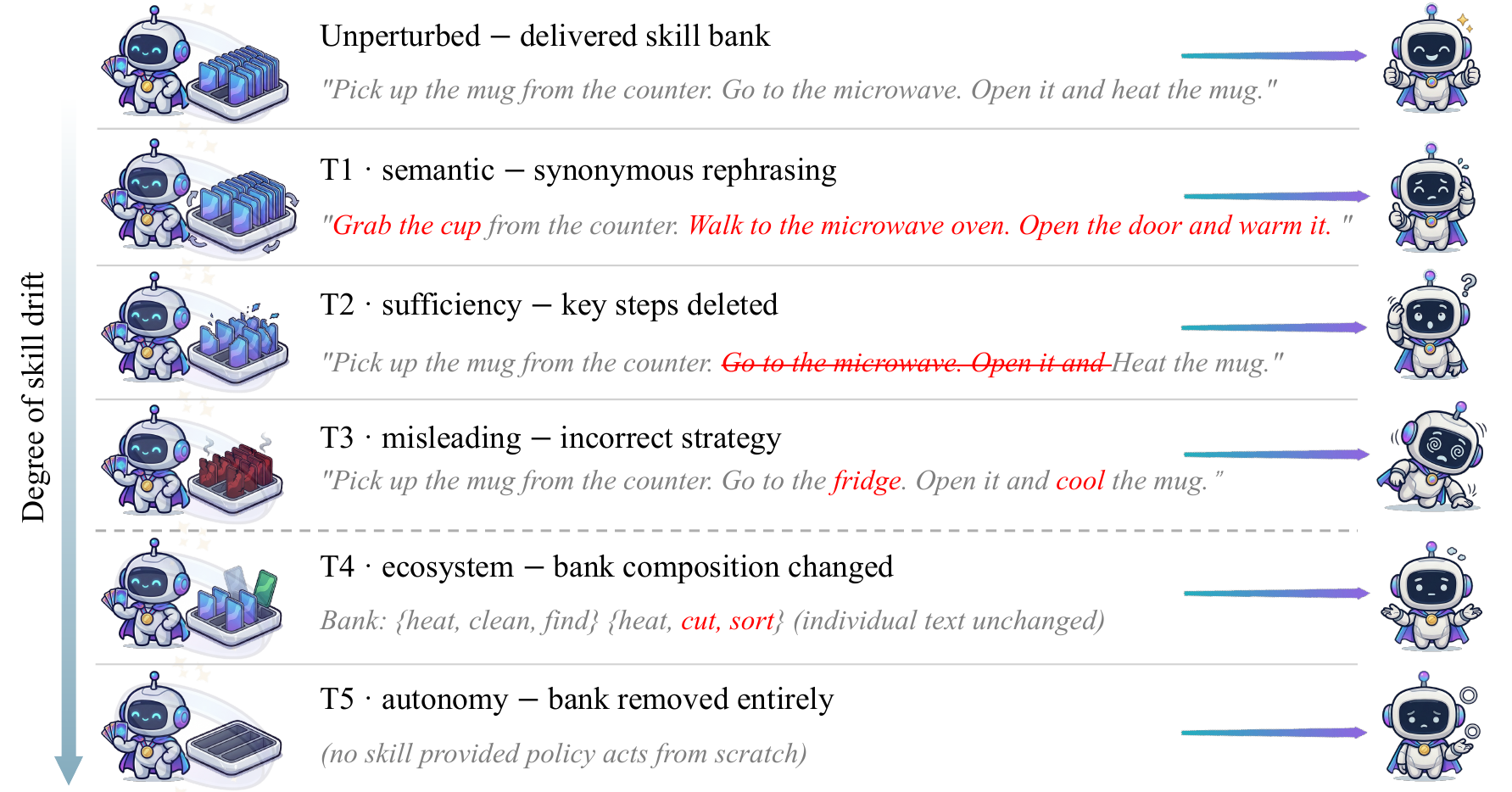}
\end{center}
\caption{The five-axis probe on one ALFWorld task. 
Each row applies a single operator to the delivered bank and reports its effect on the same skill text, with the edit marked in red.
T1--T3 alter selected skill content
under fixed retrieval and lie within $\Bnbhd(\Bankstar)$; T4 changes bank composition,
and T5 removes the bank. The task and policy remain fixed, and outcomes are schematic.}
\label{fig:probe}
\end{figure}

An external language model produces the variants from domain- and axis-specific prompts.
For T1--T3, retrieval runs once on the original bank and only the selected content fields
are replaced, preventing retrieval changes from confounding content changes. T4 and T5
alter bank composition by design. Full prompts, parameters, examples and masking details
are in Appendices~\ref{app:spec}, \ref{app:examples} and~\ref{app:masking}.

\subsection{Tracing Fragility Along Co-evolution}
\label{ssec:dose}

A two-point comparison between the base and delivered policies is difficult to compare
across settings because base performance ranges from near zero to substantial. We instead
trace checkpoints along one training trajectory while holding the delivered bank
$\Bankstar$ and its perturbed variants fixed. 
The policy checkpoint is therefore the only changing component. This design yields our proposed \emph{co-evolution fragility curve}, a within-run longitudinal diagnostic of whether fragility grows with co-evolution.
Where available, the base
(\texttt{s0}) checkpoint anchors the left endpoint; because it did not co-evolve with
$\Bankstar$, positive responses remain admissible under Proposition~\ref{prop:peak}.

We retain checkpoints from the start of the run through its peak unperturbed success rate.
Later checkpoints are not larger doses of \emph{successful} co-evolution, whereas local
training dips within the retained segment remain included
(Appendix~\ref{app:results}).

\paragraph{Absolute and normalized readings.}
We track each axis with two numbers. The absolute \emph{drop} ($\Drop$; larger means more
fragile) is the performance lost at a given checkpoint,
\begin{equation}
\Drop_a(t) \;=\; \SR(\Bankstar) - \SR(\Bankstar_a),
\label{eq:drop}
\end{equation}
where $\SR(\Bankstar)$ is the unperturbed success rate of checkpoint $t$ and
$\SR(\Bankstar_a)$ is its success rate under the axis-$a$ perturbation. The
\emph{retention performance} ($\RP$; higher is better) is competence-normalized,
\begin{equation}
\RP_a(t) \;=\; \frac{\SR(\Bankstar_a)}{\SR(\Bankstar)},
\label{eq:rp}
\end{equation}
so checkpoints of different ability compare on one scale (retaining 90\% of the score gives
$\RP = 0.9$ whether the baseline is 30\% or 80\%). Reporting both rules out the objection
that a larger drop merely reflects more headroom, which retention normalizes away.

\begin{figure}[t]
\begin{center}
\includegraphics[width=\textwidth]{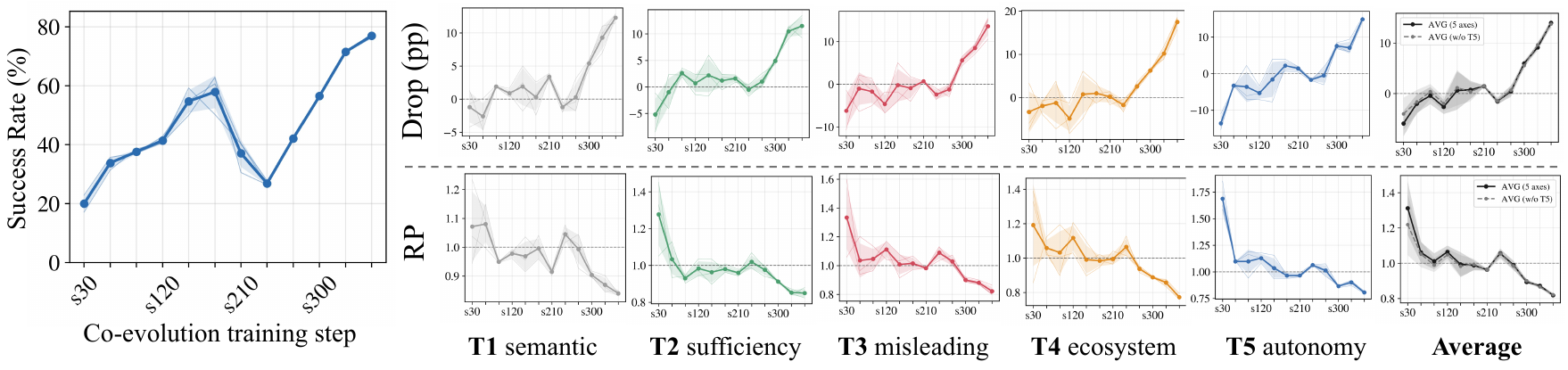}
\end{center}
\caption{Co-evolution fragility curve for \textsc{ReSkill} on ALFWorld. \textbf{Left:}
unperturbed success rate over the analyzed (rising) segment of training.
\textbf{Right, top row:} absolute drop per axis; \textbf{right, bottom row:}
retention performance $\RP$ per axis. Competence keeps rising (left) while the policy
grows more sensitive to perturbation of its own bank---$\Drop$ trends upward and its
normalized counterpart $\RP$ downward---so co-evolution buys peak performance at the cost of
fragility.}
\label{fig:dose-reskill}
\end{figure}

A minimal linear--Gaussian construction shows analytically that rising competence and fragility can emerge from the same co-evolution process (Appendix~\ref{app:proof-dose}).

\subsection{Setup and Results}
\label{ssec:results-empirics}

\paragraph{Experimental setup.}
We evaluate four skill--policy co-evolution methods over six method--environment settings:
\textsc{SkillRL}~\citep{xia2026skillrl} on ALFWorld~\citep{shridhar2020alfworld} and
WebShop~\citep{yao2022webshop}; \textsc{ReSkill}~\citep{he2026reskill} on
ScienceWorld~\citep{wang2022scienceworld} and ALFWorld;
\textsc{Evolving-RL}~\citep{fan2026evolving} on
Mind2Web~\citep{deng2023mind2web}; and \textsc{ARISE}~\citep{li2026arise} on
DeepScaleR-AMC23~\citep{luo2025deepscaler}. They share only the selection bias abstracted by
Assumption~\ref{as:accept}, so a trend common to all four tracks the paradigm rather than one
implementation. 
Appendix~\ref{app:expconfig} details method- and benchmark-specific configurations.

\paragraph{Training trend.}
The result is consistent across four methods and five benchmarks:
\textbf{as co-evolution training proceeds, the policy grows more capable with its own skill
bank but increasingly sensitive to changes in it} (Figure~\ref{fig:dose-reskill}). Unperturbed success rises from the first
to the final retained anchor in all six settings, including $5.2\%\to88.8\%$ for
\textsc{SkillRL}-ALFWorld, $16.4\%\to25.2\%$ for
\textsc{Evolving-RL}-Mind2Web and $15.5\%\to87.4\%$ for
\textsc{ARISE}-AMC23. At least one content-level axis has a significantly positive
$\Drop$ slope in five of six settings; across the eighteen axis--setting pairs, twelve are
significant---four on each of T1, T2 and T3. The effect is therefore not confined to
misleading skills: semantic and sufficiency perturbations also become costlier in multiple
settings.

The final delivered checkpoint is fragile in three concrete ways.
(i)~\emph{The one-sided prediction holds (18/18):} no content-level rewrite beats
evaluation noise in any of the eighteen axis--setting pairs, as
Proposition~\ref{prop:peak} predicts.
(ii)~\emph{The losses are substantial:} a meaning-preserving rewrite alone costs
$10.9$ points on \textsc{ReSkill}-ScienceWorld (peak $50.3\%$).
(iii)~\emph{The bank still helps under drift:} $\SR_{\min}$ stays above the no-skill anchor
in all six settings, so the failure is lost bank-derived gains, not negative skill value.
Full results are reported in Appendix~\ref{app:results}: per-anchor grids in
Tables~\ref{tab:full-skillrl-alfworld}--\ref{tab:full-arise-amc23}, heat maps in
Figures~\ref{fig:heat-skillrl-alfworld}--\ref{fig:heat-arise-amc23}, and the remaining
fragility curves in Figures~\ref{fig:dose-skillrl}--\ref{fig:dose-arise}.


The within-run trend and these delivered-pair results show one specialization pattern across
distinct methods and domains; the next section turns this dependence into \CREPE{}, a checkpoint-free audit of any delivered policy--bank pair.

\section{\CREPE{}: Auditing Co-evolved Policies for Skill Drift}
\label{sec:protocol}

The previous two sections established, theoretically and empirically, that co-evolution
leaves a policy fragile once its skills drift. What the field lacks, however, is not another
method---new update rules or acceptance tests---but a way to \emph{measure} the failure: a
shared standard that scores how well a delivered policy holds up when its bank changes,
comparably across methods. We therefore introduce \CREPE{}, a protocol that audits the one
artifact every method ships---a co-evolved policy and its delivered bank---and reports its
robustness to skill drift on a common footing.

\subsection{Design and the Shared Rewriting Suite}
\label{ssec:suite}

\CREPE{} audits only the delivered pair---a co-evolved policy and its delivered bank---plus
the shared suite, and assigns scores only to quantities with a predefined direction, so the
metric set stays closed; the design rationale (checkpoint-freeness, the direction-first rule,
and why diagnostics such as the surface-form ratio $\kappa$ of Appendix~\ref{app:kappa} are
excluded) is in Appendix~\ref{app:spec}.

The suite is the part of this paper meant to outlive it, under one hard rule:
\textbf{rewriting is performed by the released suite, not by the party under evaluation}---an evaluation whose conditions the evaluated party controls is not an evaluation.
For the content-level axes T1--T3, retrieval is run once on the unperturbed bank and only the
content of the selected skills is rewritten; the suite pins the rewriting model and its
sampling, releases the perturbed banks rather than a script, preserves bank structure, and
reports the edit strength it produced. The five requirements (R1--R5) that make this workable,
with their full rationale, are in Appendix~\ref{app:requirements}.

\subsection{A Gate and Three Robustness Metrics}
\label{ssec:metrics}
\label{ssec:collapse}

The protocol scores only T1--T3, the content-level axes within $\Bnbhd(\Bankstar)$ covered by Proposition~\ref{prop:peak}. T4 (ecosystem) and T5 (autonomy) are reported but not scored because they lack a consistent operator or direction across methods. The no-skill condition anchors the gate and resistance criterion rather than constituting a score (Appendix~\ref{app:metricdetail}).

Fix a delivered pair $(\theta^{*}, \Bankstar)$ and a test set of $N$ tasks; let
$c_i(\Bank) \in \{0,1\}$ record whether task $i$ succeeds under bank $\Bank$ and
$\SR(\Bank) = \frac{1}{N} \sum_i c_i(\Bank)$. Write $\Banknone$ for the empty bank and
$\Bankstar_a$ for axis $a \in \{\text{T1},\text{T2},\text{T3}\}$ applied to $\Bankstar$.
All comparisons are \emph{paired}: the same tasks, policy and selected skills. A
net-benefit gate admits the method, and three metrics then score it, one per content axis;
each is defined with its pass line below, and the supporting derivations (the
score-invariance bound and the fragility lower bound) are
collected in Appendix~\ref{app:metricdetail}.

\paragraph{Gate: net skill benefit.}
A bank must earn its place before its robustness matters,
\begin{equation}
\netben \;=\; \SR(\Bankstar) - \SR(\Banknone).
\label{eq:netben}
\end{equation}
A method with $\netben \le 0$ does not qualify. A small positive $\netben$, below
the noise band, is admitted but flagged: it leaves no headroom to tell degradation from
noise.

\paragraph{Metric 1: consistency rate ($\CR$), higher is better.}
A meaning-preserving rewrite (T1) should not alter task outcomes. The CR measures the fraction of outcomes unchanged by the rewrite,
\begin{equation}
\CR \;=\; \frac{1}{N} \,\big|\{\, i : c_i(\Bankstar_{\text{T1}}) = c_i(\Bankstar) \,\}\big|,
\label{eq:flip}
\end{equation}
where $\CR=1$ is ideal. Unlike $|\Delta\SR|$, $\CR$ captures opposing outcome flips that cancel at the aggregate level. To account for run-to-run variation, we use the consistency $\CR_0$ between two unperturbed runs as the baseline and declare a pass when $\CR \geq \CR_0-2\sigma_0$ (Appendix~\ref{app:metricdetail}).

\paragraph{Metric 2: resistance rate ($\RR$), higher is better.}
A judicious policy should resist an incomplete (T2) or incorrect (T3) skill by falling back on its own competence. We define the \emph{resistance rate} as the fraction of the full test set solved both without skills and under corruption:
\begin{equation}
\RR_a \;=\; \frac{1}{N}\big|\{\,i:c_i(\Banknone)=1 \wedge c_i(\Bankstar_a)=1\,\}\big|.
\label{eq:resist}
\end{equation}
Thus, $\RR_a$ is a within-setting diagnostic rather than a cross-benchmark ranking. Pass/fail is determined separately by the harmful-axis criterion
\begin{equation}
\SR(\Bankstar_a) \;\ge\; \SR(\Banknone),
\label{eq:harmpass}
\end{equation}
evaluated on paired tasks using McNemar's test. We report T2 and T3 separately because this criterion is stringent for incorrect skills but lenient for incomplete ones that retain partial value.

\paragraph{Metric 3: worst-case success rate ($\SR_{\min}$), higher is better.}
The worst-case success rate is the lowest mean performance across T1--T3:
\begin{equation}
\SR_{\min}=\min_{a\in\{\mathrm{T1},\mathrm{T2},\mathrm{T3}\}}\overline{\SR}(\Bankstar_a),
\qquad
\overline{\SR}_{\mathrm{avg}}=\frac{1}{3}\sum_a\overline{\SR}(\Bankstar_a),
\label{eq:worst}
\end{equation}
where the bar denotes the mean over repeats and $\overline{\SR}_{\mathrm{avg}}$ is reported secondarily. The gap $\SR(\Bankstar)-\SR_{\min}$ lower-bounds the fragility $\Frag$ in Definition~\ref{def:frag}; Appendix~\ref{app:metricdetail} provides the derivation.

\subsection{Verdict: An Absolute, Within-Method Pass or Fail}
\label{ssec:verdict}

\CREPE{}'s verdict is an absolute pass or fail, judged within a single method: it passes
when the gate clears ($\netben > 0$), the consistency rate stays within the $\CR_0$ noise
band, and every harmful axis satisfies Eq.~\eqref{eq:harmpass}. Every check reads a method
against its own numbers, never against another's; our headline---that the reported peak
conceals a much lower worst case---therefore ranks nothing and needs no cross-method
comparison. Ranking is a use we leave to adopters: once some pass,
$\SR_{\min}$ is the natural leaderboard key, and three raw quantities are published
regardless---the peak, the no-skill rate, and the per-axis perturbed rates.

Auditing each method on the bank it produced is a deliberate choice, not a concession: a
robustness protocol standardizes the operator and its budget, not the perturbation text, and no
shared bank exists in deployment anyway. The full argument---including why a bank too vague to
corrupt is caught by the gate---is in Appendix~\ref{app:fairness}.

\subsection{Main Result: the Report Card}
\label{ssec:setup}
\label{ssec:reportcard}

We audit four co-evolution methods across the six method--environment settings of
Section~\ref{ssec:results-empirics}, running \CREPE{} on each final delivered pair---not any
intermediate checkpoint---with perturbed banks from the shared rewriting suite. To keep the comparison fair, the peak we report is our
own re-evaluation of each pair under a single fixed protocol, not the number in the original
paper; because backbones differ, absolute numbers compare ``method $+$ backbone'' rather than
the method alone (Appendix~\ref{app:expconfig}). Table~\ref{tab:reportcard} is the audit: the
top block gives the headline gap between the peak and the worst case once the bank drifts, and
every other row is what \CREPE{} adds---how that same delivered system behaves once it no
longer controls the wording of its bank.

\begin{table}[t]
\caption{\CREPE{} report card. The top row is the peak these methods report, re-evaluated
under one fixed protocol; the rows below show how much of it survives once the delivered bank
drifts. The headline drop (shaded) is 5--13 points and erases 52--93\% of the skill benefit
over the no-skill baseline, and the consistency rate stays below its noise floor $\CR_0$
throughout---so the reported peak overstates the competence that actually deploys. Details in
Appendix~\ref{app:results}.}
\label{tab:reportcard}
\begin{center}
\scalebox{0.8}{
\begin{tabular}{lcccccc}
\toprule
 & \multicolumn{2}{c}{SkillRL} & \multicolumn{2}{c}{ReSkill} & Evolving-RL & ARISE \\
\cmidrule(lr){2-3} \cmidrule(lr){4-5} \cmidrule(lr){6-6} \cmidrule(lr){7-7}
Metric & ALFWorld & WebShop & ScienceWorld & ALFWorld & Mind2Web & AMC23 \\
\midrule
peak success rate $\SR(\Bankstar)$ & 88.8 & 52.4 & 50.3 & 76.9 & 25.2 & 87.4 \\
worst case $\SR_{\min}$            & 78.6 & 47.4 & 39.4 & 64.3 & 20.3 & 80.9 \\
\rowcolor{codahighlight} \textbf{peak$\to$worst drop} & \textbf{10.2} & \textbf{5.0} & \textbf{10.9} & \textbf{12.6} & \textbf{4.9} & \textbf{6.5} \\
\midrule
\rowcolor{gray!15} \multicolumn{7}{l}{\emph{Per-axis success rate}} \\
\quad T1 reword                    & 78.6 & 51.5 & 39.4 & 64.3 & 23.2 & 82.5 \\
\quad T2 incomplete                & 79.3 & 47.4 & 47.7 & 66.6 & 20.3 & 81.5 \\
\quad T3 misleading                & 79.5 & 51.7 & 47.7 & 65.4 & 20.4 & 80.9 \\
\quad no-skill anchor $\SR(\Banknone)$ & \underline{74.5} & \underline{45.9} & \underline{34.2} & \underline{63.3} & \underline{15.8} & \underline{75.3} \\
\midrule
\rowcolor{gray!15} \multicolumn{7}{l}{\emph{Robustness metrics} (higher better; not success rates)} \\
\quad consistency rate $\CR$       & 87.0 & 83.2 & 86.1 & 84.2 & 77.5 & 82.1 \\
\quad\quad \textcolor{gray}{noise floor $\CR_0$} & \textcolor{gray}{92.9} & \textcolor{gray}{91.9} & \textcolor{gray}{93.6} & \textcolor{gray}{92.2} & \textcolor{gray}{94.4} & \textcolor{gray}{93.7} \\
\quad resistance rate $\RR_{\text{T2}}$ & 70.2 & 37.8 & 31.0 & 60.6 & 12.2 & 63.5 \\
\quad resistance rate $\RR_{\text{T3}}$ & 71.2 & 40.0 & 33.5 & 58.9 & 11.4 & 65.4 \\
\bottomrule
\end{tabular}}
\end{center}
\end{table}

\paragraph{The reported peak overstates deployable competence.}
Across all six settings the worst-case success rate under content-level drift trails the peak
by 4.9--12.6 points (about 5--13): 4.9 on \textsc{Evolving-RL}-Mind2Web, 5.0 on
\textsc{SkillRL}-WebShop, 6.5 on \textsc{ARISE}-AMC23 and 10.2--12.6 on the other three. Only
the peak is customarily reported. Because the absolute drop scales with how much the bank
helps, we normalize by that headroom: drift consumes 52--93\% of the skill benefit in every
setting. The worst-case rate nonetheless stays above the no-skill baseline throughout, so the
failure is not that skills hurt but that most of the advertised gain is contingent on the
bank's exact wording and disappears once it drifts. What a deployment can rely on is therefore
nearer this worst case than the peak, by the majority of the benefit skills appear to add.

\paragraph{The worst case is often the mildest change.}
In three of the six settings the T1 (meaning-preserving) rate falls \emph{below} the harsher
T2 and T3 rates (e.g.\ \textsc{SkillRL}-ALFWorld 78.6 vs.\ 79.3/79.5), costing up to 12.6
points relative to the peak (\textsc{ReSkill}-ALFWorld 76.9$\to$64.3)---evidence that the
policy has overfit to the bank's phrasing rather than its meaning. The consistency rate makes
the churn explicit: at $77.5$--$87.0$ it sits below its noise floor $\CR_0 = 91.9$--$94.4$, so
$13$--$23\%$ of tasks flip outcome under a meaning-preserving rewrite, against $6$--$8\%$ under
decoding noise alone; the consistency criterion therefore fails in every setting---and with it
the verdict---while the net-benefit gate and the harmful-axis criterion pass. \CREPE{} does not
declare these methods useless: it quantifies how much of each reported peak is co-evolved with
the delivered bank, precisely the quantity current evaluation leaves unmeasured.

\section{Related Work}
\label{sec:related}

\paragraph{Skill--policy co-evolution.}
A growing family of methods co-evolves a skill library with the policy, including
\textsc{SkillRL}~\citep{xia2026skillrl}, \textsc{ReSkill}~\citep{he2026reskill},
\textsc{Skill1}~\citep{shi2026skill1}, \textsc{ARISE}~\citep{li2026arise},
\textsc{SAPO}~\citep{zhang2026co}, and
\textsc{SkillMaster}~\citep{yang2026skillmaster}.
These methods differ in update cadence, library growth, validation gates, and skill injection
(Table~\ref{tab:framework}), yet share the same evaluation premise: the co-evolved policy is
tested on its delivered bank. We introduce no new member of this family; instead, we examine
this shared premise and its consequences under deployment-time skill drift.

\paragraph{Agent robustness under skill drift.}
Robustness evaluations for language agents typically perturb the task side of the
interaction---instructions, observations, tool outputs, or environment dynamics.
We instead perturb the agent's own procedural knowledge, which task-side benchmarks treat as
fixed. Recent work documents that skill libraries drift and that agent performance can
degrade accordingly~\citep{zhang2026library,chen2026vce,liu2026well,chen2026scaling}, that
skills may harm rather than help~\citep{tank2026regression,yu2026skill}, and that
self-evolution is not always faithful~\citep{zhao2026large,lv2026can}.
Other approaches repair skills post hoc or construct more drift-tolerant
skills~\citep{liu2026skillrevise,park2026learning,shao2026lifting}; by contrast, we audit the
delivered policy's dependence on whichever bank it encounters.
This setting also resembles zero-shot coordination~\citep{hu2020otherplay}, with the skill
bank acting as a co-evolved partner that may change after training.
Appendix~\ref{app:related} provides further discussion.

\section{Conclusion}
\label{sec:conclusion}

This work shifts the evaluation of skill--policy co-evolution from performance on a single delivered pair to robustness under realistic deployment-time skill drift. Our analysis identifies cooperative acceptance as a source of dependence on the co-evolved bank, while the \emph{co-evolution fragility curve} shows that this dependence can strengthen over training. \CREPE{} operationalizes these findings by auditing whether a delivered policy remains consistent under benign rewrites and resists incomplete or misleading skills. The delivered peak is therefore informative but insufficient: future methods should preserve its gains while learning to interpret, reject, and adapt to skills as the bank evolves.

\section*{AI use statement}

Large language models play two distinct roles in this work. As a \emph{component of the
method}, the \CREPE{} rewriting suite calls \texttt{gpt-4o-2024-08-06} (temperature $0.7$,
with version and sampling parameters pinned as specified in Appendix~\ref{app:spec}) to
generate the perturbed skill banks that serve as evaluation stimuli; because these banks are
released, reproducing our evaluation does not require re-running the model. As an
\emph{authoring aid}, LLMs assisted only with language polishing, grammar, and style. All
research questions, methodology, experiments, analyses, and conclusions were conceived,
implemented, and verified by the authors, who take full responsibility for the content of
this paper.

\section*{Ethics statement}

The \CREPE{} suite deliberately produces incomplete (T2) and misleading (T3) procedural
instructions in order to measure how a delivered policy withstands skill drift. These
corrupted skills are generated and consumed entirely within the benchmark harness, are never
deployed or shown to real users, and are released alongside the suite so that their content
is transparent and auditable. The work involves no human participants, collects no personally
identifiable information, and introduces no new sensitive data; we see no direct risks to
privacy, fairness, or legal compliance, and the work complies with the ICLR Code of Ethics.

\section*{Reproducibility statement}

We take several concrete measures to make our results reproducible.
\begin{itemize}[leftmargin=*]
\item \textbf{Rewriting suite and protocol:} the rewriting model, its version, and all
sampling parameters are pinned (Section~\ref{ssec:suite}), and the domain contexts, layered
prompts, and resulting perturbed banks are released, so the evaluation can be reproduced
without re-running the generative model.
\item \textbf{Environments and data:} all benchmarks used (ALFWorld, WebShop, ScienceWorld,
Mind2Web, and DeepScaleR-AMC23) are publicly available; the test splits, episode counts, and
evaluation configuration are documented in Section~\ref{ssec:setup} and
Appendix~\ref{app:expconfig}, together with the released configuration.
\end{itemize}
Together, these measures allow independent researchers to reproduce our results and verify
the claims made in this paper.

\bibliography{iclr2027_conference}
\bibliographystyle{iclr2027_conference}

\appendix

\section{Limitations}
\label{app:limitations}

We note three limitations of this work.

\begin{enumerate}[leftmargin=*]
\item[\textbf{L1}] \emph{Rewriting is not bit-reproducible.} A language model rewriting a
skill is not a deterministic operator: pinned model versions, pinned sampling, and released
perturbed banks reduce the variance a re-runner faces but do not remove it. If the bank was
written by the same model family that later rewrites it, the measured fragility may be
optimistically low.

\item[\textbf{L2}] \emph{We diagnose but do not fix.} We expose the fragility and provide a
measurement tool, but propose no mitigation. The natural direction is the distributional
objective of Eq.~\eqref{eq:decomposition}---training against a distribution of bank
perturbations rather than a single fixed bank---but turning it into a practical method, and
keeping the net-benefit gate satisfied, is future work.

\item[\textbf{L3}] \emph{Scale is unstudied.} Whether broader exposure---more benchmarks,
larger banks, more data---reduces sensitivity to wording or merely raises the peak from
which the worst case falls away is beyond our design, and is the most interesting open
question beyond the fix itself.
\end{enumerate}

\section{Extended Related Work}
\label{app:related}

\paragraph{Skill--policy co-evolution.}
A growing family of methods lets an agent distill task experience into reusable skills and
then trains the policy on the enlarged bank~\citep{liu2026path, zou2026llm, yang2026skillforge, feng2026coskill}.
\textsc{SkillRL}~\citep{xia2026skillrl} distills successful trajectories into skills that are
added to a monotonically growing library and gated by reward on sampled rollouts.
\textsc{ReSkill}~\citep{he2026reskill} caps the library at a fixed size and uses Thompson
sampling to decide which incumbent a new candidate should replace.
\textsc{Skill1}~\citep{shi2026skill1} jointly optimizes the policy and bank end-to-end,
admitting skills from trajectories that already succeed.
\textsc{ARISE}~\citep{li2026arise} evolves a two-tier hierarchy of intrinsic skills under a
hierarchical reward.
\textsc{SAPO}~\citep{zhang2026co} admits a skill only when its estimated marginal
utility to the current policy is positive.
\textsc{SkillMaster}~\citep{yang2026skillmaster} lets the policy itself score candidates,
retaining high-scoring ones into a managed library.
These methods differ in cadence, library growth, validation gate, and injection mechanism
(Table~\ref{tab:framework}; Appendix~\ref{app:framework} gives the full classification).
What they do not differ in is how success is reported: the co-evolved policy is run on the
delivered bank and a task success rate is given.
Surveys chart this space~\citep{zhou2026comprehensive,gao2025survey,ren2026self}, and
recent variants push co-evolution to test-time~\citep{mao2026learning},
self-play~\citep{huang2026skill}, and deployment-time meta-learning~\citep{xia2026metaclaw}.
We contribute no new member of this family. Our object of study is the family
itself---specifically, the one premise that survives all of their design disagreements.

\paragraph{Agent robustness and adversarial co-evolution.}
Work on robustness for language agents perturbs the task side of the
interaction---instructions, observations, tool outputs, environment
dynamics---and asks whether performance survives. We perturb a different thing: the
agent's own store of procedural knowledge, which no task-side benchmark touches
because it is normally treated as part of the agent rather than part of the
environment. Concurrently, a cluster of works documents that skill libraries
drift and that agents degrade under that drift~\citep{zhang2026library,chen2026vce,liu2026well,chen2026scaling},
that skills can hurt as often as they help~\citep{tank2026regression,yu2026skill}, and that
self-evolution is not always faithful~\citep{zhao2026large,lv2026can}. These diagnose the
symptom; none ties it to the co-evolution objective or measures it on the delivered
artifact. SkillRevise~\citep{liu2026skillrevise} repairs skills post hoc, and
GenDa~\citep{park2026learning} with programmatic skill induction~\citep{shao2026lifting} build
drift-tolerant skills from the skill side; we instead audit the policy's dependence on
whatever bank it is handed. The closest conceptual relative is zero-shot
coordination~\citep{hu2020otherplay}, where a policy trained with one partner must
act with another it never co-evolved with; our setting is the single-agent analog,
with the skill library in the role of the partner that changed. A separate line
evolves an adversary alongside the learner, in game-theoretic heuristic
discovery~\citep{ke2026game} or attacker--defender safety
training~\citep{magic2026,pan2025advevo}, and reports robustness \emph{gains}. This is not a
counterexample to Proposition~\ref{prop:peak} but a confirmation of its premise
read in reverse: those methods accept a candidate precisely when it \emph{lowers}
the current policy's return, so the sign in Assumption~\ref{as:accept} is flipped
and the conclusion flips with it. Appendix~\ref{app:adversarial} works this through,
and the comparison is instructive rather than competitive: it locates the fragility
in the acceptance rule rather than in co-evolution as such.

\paragraph{Standardized perturbation suites.}
\CREPE{} belongs to a tradition of fixed perturbation suites applied to models their
designers did not build: corruption benchmarks in
vision~\citep{hendrycks2019benchmarking} and behavioral test suites in
NLP~\citep{ribeiro2020checklist} both consist of a released operator, a reporting
convention, and the finding that published systems fail more than their headline
numbers suggest. We follow that template on the skill side. One difference must be
stated rather than glossed over: those operators are deterministic algorithms that
reproduce bit-for-bit, whereas rewriting a skill with a language model does not. We narrow the gap: we pin the rewriting model, its version and its sampling parameters; we
ship the domain context so nothing is left to the user; and we release the perturbed banks
alongside the suite (Section~\ref{ssec:suite}). We do not close it, though---
Section~\ref{sec:conclusion} lists it first among the limitations.


\section{Classification of Existing Methods}
\label{app:framework}

The four axes of Table~\ref{tab:framework} are defined as follows. \emph{Update
cadence} determines how the two steps interleave, from strictly alternating through
near-synchronous to jointly optimized. The \emph{validation gate} determines how a
candidate bank is admitted: no test at all, an improvement test on sampled rollouts,
or a posterior test such as Thompson sampling. \emph{Library growth} is either
monotone, so that skills accumulate, or capped, so that admission displaces an
incumbent. The \emph{injection mechanism} determines what reaches the context at
inference: a batch selected by task category, or a small subset triggered per step.

We now place each surveyed method on these axes and state its acceptance rule---the
property that Assumption~\ref{as:accept} abstracts. In every case the gate admits a
candidate bank only when it does not lower (and usually strictly raises) the current
policy's return, however differently that test is implemented.

\begin{itemize}[leftmargin=*]
\item \textbf{\textsc{SkillRL}}~\citep{xia2026skillrl}: alternating cadence, monotone
growth, batch top-$k$ injection. A distilled skill is admitted whenever it does not
reduce reward on sampled rollouts and is never removed thereafter, so the retained bank
is the reward-non-decreasing one.
\item \textbf{\textsc{ReSkill}}~\citep{he2026reskill}: near-synchronous cadence, library
capped at eight, per-step injection. A candidate is validated by Thompson sampling and
kept only if its posterior return improves on the incumbent it would displace.
\item \textbf{\textsc{SAPO}}~\citep{zhang2026co}: near-synchronous cadence, two-tier capped
growth, reranked injection. A skill enters only when its estimated marginal utility to
the current policy is positive.
\item \textbf{\textsc{Skill1}}~\citep{shi2026skill1}: joint optimization, monotone growth,
log-probability injection. Skills are distilled from trajectories that already succeed
($r{=}1$) and trained jointly with the policy, so what is retained is exactly what
supported a return the policy did not lose.
\item \textbf{\textsc{ARISE}}~\citep{li2026arise}: joint optimization, two-tier growth,
log-probability injection. Intrinsic skills are evolved under a hierarchical reward and
retained on improvement of that reward.
\item \textbf{\textsc{SkillMaster}}~\citep{yang2026skillmaster}: near-synchronous cadence,
policy-managed growth and injection. The policy itself scores candidate skills and
retains the high-scoring ones, i.e.\ those that raise its own return.
\item \textbf{Evolving-RL}~\citep{fan2026evolving}: joint optimization, monotone growth, batch injection. A
skill is admitted end-to-end when it clears a reward threshold on the current policy.
\item \textbf{COSPLAY}~\citep{wu2026co}: near-synchronous cadence, lifecycle-managed growth, per-step
injection. Skills pass a contract-verification step and are retained over their
lifecycle while they continue to serve the policy's return.
\end{itemize}

The gates differ in form---an improvement test on rollouts, a Thompson posterior, a
marginal-utility check, a reward threshold, a policy score, a contract---but all favor
skills that serve the current policy. Assumption~\ref{as:accept} idealizes this shared
selection bias. Proposition~\ref{prop:peak} is model-agnostic but additionally requires the
candidate-coverage condition of Assumption~\ref{as:coverage}; it therefore applies
conditionally rather than automatically to the eight surveyed methods.

\section{Proofs and the Minimal Linear--Gaussian Construction}
\label{app:proofs}

\subsection{Setup and Notation}
\label{app:setup}

We restate the objects used throughout. A task $x$ is drawn from $D$, a skill bank
$\Bank$ lives in $\mathcal{S}$, and the return of a policy is
$\Jobj(\theta, \Bank) = \Exp_{x \sim D}[R(\pol(\cdot \mid x, \Bank))]$ as in
Eq.~\eqref{eq:objective}. At round $t$ the candidate set
$\mathcal{C}_t \subseteq \mathcal{S}$ collects the banks the method's proposal
mechanism could have produced and its gate could have retained;
Assumption~\ref{as:accept} states $\Bank_t \in \arg\max_{\Bank \in \mathcal{C}_t}
\Jobj(\theta_t, \Bank)$. The content-level neighborhood $\Bnbhd(\Bank)$ of
Definition~\ref{def:nbhd} contains a specified family of textual rewrites, includes $\Bank$
itself and excludes the empty bank. Assumption~\ref{as:coverage} requires
$\Bnbhd(\Bankstar)\subseteq\mathcal{C}_T$ only at delivery. The response
set $\mathcal{R}$ and fragility $\Frag$ are as in Definition~\ref{def:frag}. We write
$(\theta^{*}, \Bankstar) = (\theta_T, \Bank_T)$ for the delivered pair and drop the
round index when it is not needed.

\subsection{Proof of Proposition~\ref{prop:peak}}
\label{app:proof-peak}

\begin{proof}
For the first claim, Assumption~\ref{as:accept} at round $t = T$ gives
$\Jobj(\theta^{*}, \Bankstar) = \max_{\Bank \in \mathcal{C}_T} \Jobj(\theta^{*}, \Bank)$,
hence $\Jobj(\theta^{*}, \Bank') \le \Jobj(\theta^{*}, \Bankstar)$ for every
$\Bank' \in \mathcal{C}_T$. Since $\Bnbhd(\Bankstar) \subseteq \mathcal{C}_T$ by
Assumption~\ref{as:coverage}, the inequality holds in particular for
every $\Bank' \in \Bnbhd(\Bankstar)$, so each element
$\Jobj(\theta^{*}, \Bank') - \Jobj(\theta^{*}, \Bankstar)$ of
$\mathcal{R}(\theta^{*}, \Bankstar)$ is non-positive, which is
Eq.~\eqref{eq:nonpositive}.

For the second claim, $\Frag(\theta^{*}, \Bankstar) = 0$ means
$\min_{\Bank' \in \Bnbhd(\Bankstar)} \Jobj(\theta^{*}, \Bank') = \Jobj(\theta^{*}, \Bankstar)$.
Combined with the upper bound just established, this forces
$\Jobj(\theta^{*}, \Bank') = \Jobj(\theta^{*}, \Bankstar)$ for all
$\Bank' \in \Bnbhd(\Bankstar)$, i.e.\ the return is constant on the neighborhood.

For the third claim, Assumption~\ref{as:accept} constrains $\Jobj(\theta_0, \cdot)$
only if $\theta_0$ was the policy against which the bank step was solved. If it was
not, no constraint relates $\Jobj(\theta_0, \Bankstar)$ to
$\Jobj(\theta_0, \Bank')$ for $\Bank' \in \Bnbhd(\Bankstar)$, so
$\max \mathcal{R}(\theta_0, \Bankstar) > 0$ is consistent with the premise. An
explicit instance in which this maximum is strictly positive is given in
Appendix~\ref{app:proof-dose}.
\end{proof}

\begin{remark}[The one-sidedness is essential]
\label{rem:onesided}
Because $\Bankstar \in \Bnbhd(\Bankstar)$, the inequality $\Frag \ge 0$ holds for
every policy and carries no information about co-evolution. The content of
Proposition~\ref{prop:peak} is the \emph{sign} statement
$\mathcal{R} \subseteq (-\infty, 0]$: for a co-evolved policy the response set has no
positive part. Reporting only the magnitude of the worst-case drop therefore cannot
distinguish a co-evolved policy from a merely brittle one, which is why the
per-axis results in Section~\ref{sec:empirics} report the full sign pattern rather
than magnitude alone.
\end{remark}

\begin{remark}[Why removing skills is excluded]
\label{rem:noskill}
Removing every skill can \emph{raise} a co-evolved policy's return, which would
appear to contradict Eq.~\eqref{eq:nonpositive}. It does not: the empty bank lies
outside $\Bnbhd(\Bankstar)$ by Definition~\ref{def:nbhd}, and the exclusion is not a
convenience. Assumption~\ref{as:accept} says only that $\Bankstar$ beats what the
proposal mechanism \emph{offered}, and no co-evolution proposal mechanism offers to
delete the bank; the empty bank was therefore never rejected in favor of
$\Bankstar$ and nothing constrains its return. The theory's domain thus coincides
with the protocol's decision to treat removal as an anchor rather than a perturbation
(Section~\ref{ssec:collapse}).
\end{remark}

\subsection{Construction~\ref{con:minimal} and Proof of Proposition~\ref{prop:dose}}
\label{app:proof-dose}

Propositions~\ref{prop:peak} and~\ref{prop:marginal} bound the sign of the response
but say nothing about the magnitude of $\Frag$ or its evolution during training. We
exhibit the smallest instance containing both ingredients of co-evolution---a policy
that reads the bank, and a bank selected under Assumption~\ref{as:accept} to help the
current policy---in which both are available in closed form.

\begin{construction}[Minimal linear--Gaussian instance]
\label{con:minimal}
A binary label is predicted from two independent noisy channels: a \emph{semantic}
channel whose strength is fixed by the environment, and a \emph{form} channel whose
salience the bank step controls. The policy is a linear classifier trained as a best
response, and the bank step raises form salience under Assumption~\ref{as:accept}.
\end{construction}

\paragraph{The instance.}
A label $y \in \{-1, +1\}$ is uniform. The policy observes two scalar channels,
\begin{equation}
x_r \;=\; \mu_r\, y + \varepsilon_r, \quad \varepsilon_r \sim \mathcal{N}(0, \sigma_r^2),
\qquad
x_u \;=\; \mu_u\, y + \varepsilon_u, \quad \varepsilon_u \sim \mathcal{N}(0, \sigma_u^2),
\label{eq:channels}
\end{equation}
with independent noises. The \emph{semantic} channel $x_r$ has strength $\mu_r$ fixed
by the environment: it represents what the skill means, which no rewrite alters. The
\emph{form} channel $x_u$ carries the label only through the particular surface
encoding the bank supplies, and its salience $\mu_u$ is the quantity the bank step
controls. The policy is the linear rule $\mathrm{sign}(w \cdot x)$ with
$w = (w_r, w_u)$. Conditioned on $y$ we have
$w \cdot x \sim \mathcal{N}(y\,m, s^2)$ with $m = w_r \mu_r + w_u \mu_u$ and
$s^2 = w_r^2 \sigma_r^2 + w_u^2 \sigma_u^2$, so the return is
\begin{equation}
\Jobj(w, \mu_u) \;=\; \normcdf\!\left(\frac{w_r \mu_r + w_u \mu_u}{\sqrt{w_r^2 \sigma_r^2 + w_u^2 \sigma_u^2}}\right),
\label{eq:toyreturn}
\end{equation}
where $\normcdf$ is the standard normal CDF.

\paragraph{The dynamics.}
The \emph{policy step} maximizes Eq.~\eqref{eq:toyreturn} over $w$, which is the Fisher
discriminant $w \propto (\mu_r/\sigma_r^2,\; \mu_u/\sigma_u^2)$. The \emph{bank step}
follows Assumption~\ref{as:accept}: since Eq.~\eqref{eq:toyreturn} is increasing in
$\mu_u$ whenever $w_u > 0$, any candidate that raises form salience is accepted, so
$\mu_u^{(t)}$ is non-decreasing in $t$. Define the two squared signal-to-noise ratios
\begin{equation}
a \;=\; \frac{\mu_r^2}{\sigma_r^2} \quad (\text{fixed}),
\qquad
b^{(t)} \;=\; \frac{\big(\mu_u^{(t)}\big)^2}{\sigma_u^2} \quad (\text{non-decreasing; the \emph{dose}}).
\label{eq:snr}
\end{equation}

\paragraph{Clean and perturbed return.}
Substituting the Fisher weights into Eq.~\eqref{eq:toyreturn} gives numerator
$\mu_r^2/\sigma_r^2 + \mu_u^2/\sigma_u^2 = a + b$ and denominator
$\sqrt{\mu_r^2/\sigma_r^2 + \mu_u^2/\sigma_u^2} = \sqrt{a+b}$, hence
\begin{equation}
\Jobj_{\mathrm{clean}}^{(t)} \;=\; \normcdf\!\left(\sqrt{a + b^{(t)}}\right).
\label{eq:toyclean}
\end{equation}
A meaning-preserving rewrite changes the surface encoding while leaving the semantics
intact, so the direction the policy learned to read no longer predicts the label: in
the model, $\mu_u \to 0$ in the data while the weights, fixed at deployment, and the
channel noise are unchanged. The numerator becomes $a$ and the denominator is
unchanged, giving
\begin{equation}
\Jobj_{\mathrm{pert}}^{(t)} \;=\; \normcdf\!\left(\frac{a}{\sqrt{a + b^{(t)}}}\right),
\qquad
\Frag^{(t)} \;=\; \normcdf\!\left(\sqrt{a + b^{(t)}}\right) - \normcdf\!\left(\frac{a}{\sqrt{a + b^{(t)}}}\right).
\label{eq:toyfrag}
\end{equation}

\begin{proposition}[Dose--response, in Construction~\ref{con:minimal}]
\label{prop:dose}
Both peak return and fragility $\Frag^{(t)}$ are strictly increasing in the dose
$b^{(t)}$. A policy that does not read the form channel has $\Frag^{(t)} = 0$ for all
$t$.
\end{proposition}

\begin{proof}[Proof of Proposition~\ref{prop:dose}]
Write $\varphi_c(b) = \normcdf(\sqrt{a+b})$ and $\varphi_p(b) = \normcdf(a/\sqrt{a+b})$
for $b \ge 0$, so that $\Jobj_{\mathrm{clean}} = \varphi_c(b)$ and
$\Frag = \varphi_c(b) - \varphi_p(b)$. The map $b \mapsto \sqrt{a+b}$ is strictly
increasing and $\normcdf$ is strictly increasing, so $\varphi_c$ is strictly
increasing: peak return rises with the dose. The map $b \mapsto a/\sqrt{a+b}$ is
strictly decreasing for $a > 0$, since the numerator is constant and the denominator
strictly increases, so $\varphi_p$ is strictly decreasing. Hence
$\Frag = \varphi_c - \varphi_p$ is a sum of a strictly increasing function and the
negative of a strictly decreasing function, and is therefore strictly increasing in
$b$. Because $b^{(t)}$ is non-decreasing along the trajectory and strictly increases
at any accepted step, both $\Jobj_{\mathrm{clean}}^{(t)}$ and $\Frag^{(t)}$ increase
along training.

For the final claim, a policy with $w_u = 0$ never reads the form channel; then
$m = w_r \mu_r$ and $s^2 = w_r^2 \sigma_r^2$ both in the clean and in the perturbed
condition, so $\Jobj_{\mathrm{clean}} = \Jobj_{\mathrm{pert}} = \normcdf(\sqrt{a})$
and $\Frag = 0$ for every $t$.
\end{proof}

\begin{corollary}[The sign reversal of Proposition~\ref{prop:peak}, realized]
\label{cor:signflip}
Let $\theta_0$ have a misaligned form weight, $w_u \mu_u < 0$, as may happen for a
policy that never co-evolved to this bank. Then the numerator of Eq.~\eqref{eq:toyreturn}
satisfies $w_r \mu_r + w_u \mu_u < w_r \mu_r$, while the perturbed numerator is
$w_r \mu_r$ and the denominator is unchanged; hence
$\Jobj_{\mathrm{pert}} > \Jobj_{\mathrm{clean}}$ and
$\max \mathcal{R}(\theta_0, \Bankstar) > 0$. Destroying a channel the policy was
reading incorrectly improves it, which is the admissible case of
Proposition~\ref{prop:peak}.
\end{corollary}

\begin{remark}[Modeling choice in the policy step]
\label{rem:beststep}
The policy step above is idealized as an exact best response, which is what makes
Eq.~\eqref{eq:toyfrag} exact. Modeling it as gradient ascent is closer to practice
but makes the closed form asymptotic, in which case Proposition~\ref{prop:dose} reads
``increasing in the limit'' rather than ``strictly increasing''. We keep the
best-response form for its clean closed form and note the gradient-ascent reading here.
\end{remark}

\begin{remark}[What the construction does not explain]
\label{rem:toylimit}
Within Construction~\ref{con:minimal} a harmful rewrite corresponds to reversing the
form channel, $\mu_u \to -\mu_u$, which yields numerator $a - b$ and hence a larger
degradation than the meaning-preserving rewrite for every $b > 0$. The instance
therefore predicts that harmful edits hurt more than meaning-preserving ones. We do
not test this ordering, and flag it as a limitation of the instance: it isolates what
a two-channel account can capture and leaves finer orderings among the axes open.
\end{remark}

\subsection{Proof of Proposition~\ref{prop:marginal} and the Invariance of the
Point-Mass Objective}
\label{app:proof-marginal}

\begin{proof}[Proof of Proposition~\ref{prop:marginal}]
Add and subtract $\Jobj(\theta, \Bankstar)$ inside the expectation,
\[
\Exp_{\Bank \sim P}\!\left[\Jobj(\theta, \Bank)\right]
= \Exp_{\Bank \sim P}\!\left[\Jobj(\theta, \Bankstar) - \big(\Jobj(\theta, \Bankstar) - \Jobj(\theta, \Bank)\big)\right]
= \Jobj(\theta, \Bankstar) - \Exp_{\Bank \sim P}\!\left[\Jobj(\theta, \Bankstar) - \Jobj(\theta, \Bank)\right],
\]
which is Eq.~\eqref{eq:decomposition} since $\Jobj(\theta, \Bankstar)$ does not depend
on $\Bank$. Under Assumptions~\ref{as:accept} and~\ref{as:coverage}, each integrand
$\Jobj(\theta^{*}, \Bankstar) - \Jobj(\theta^{*}, \Bank)$ is non-negative by
Proposition~\ref{prop:peak}, so the average degradation is non-negative; it is bounded
above by $\Frag(\theta^{*}, \Bankstar)$ because every integrand is. Taking
$P = \delta_{\Bankstar}$ makes the integrand vanish identically, recovering
$\Jobj(\theta, \Bankstar)$.
\end{proof}

The companion statement referred to in Section~\ref{ssec:prop3} says that the training
objective cannot see fragility at all. It is a statement about the \emph{objective},
not about a particular network, and unlike Proposition~\ref{prop:marginal} it needs a
hypothesis on the policy class.

\begin{proposition}[The point-mass objective is blind to fragility]
\label{prop:blind}
Suppose the policy class is rich enough that behavior at $\Bankstar$ can be held
fixed while behavior elsewhere on $\Bnbhd(\Bankstar)$ varies: for any
$\theta_1$ there exists $\theta_2$ with
$\pi_{\theta_1}(\cdot \mid x, \Bankstar) = \pi_{\theta_2}(\cdot \mid x, \Bankstar)$ for all $x$ and
arbitrary behavior on $\Bnbhd(\Bankstar) \setminus \{\Bankstar\}$. Then the training
objective $\Jobj(\cdot, \Bankstar)$ is invariant to $\Frag$: its maximizers span the
full achievable range of fragility, so training provides no signal on it.
\end{proposition}

\begin{proof}
Let $\theta_1$ be any maximizer of $\Jobj(\cdot, \Bankstar)$. By the
richness hypothesis there is $\theta_2$ agreeing with $\theta_1$ at $\Bankstar$ for
every task, hence with $\Jobj(\theta_2, \Bankstar) = \Jobj(\theta_1, \Bankstar)$, so
$\theta_2$ is also a maximizer, while its behavior on
$\Bnbhd(\Bankstar) \setminus \{\Bankstar\}$ is arbitrary and therefore so is
$\min_{\Bank'} \Jobj(\theta_2, \Bank')$. Consequently
$\Frag(\theta_2, \Bankstar)$ ranges over the full achievable interval as $\theta_2$
varies within the set of maximizers: the objective does not distinguish them, and its
gradient with respect to any parameter affecting only off-$\Bankstar$ behavior
vanishes.
\end{proof}

The hypothesis is what makes precise that $\Jobj(\cdot, \Bankstar)$ cannot distinguish
a fragile policy from a robust one. We do not claim both policies are reachable under a
shared-weight parameterization, where behavior at $\Bankstar$ and on its neighborhood
are coupled; this is why the main text relies on Eq.~\eqref{eq:decomposition} alone.

\begin{corollary}[Hedging in Construction~\ref{con:minimal}]
\label{cor:hedge}
Take $P$ to leave the form channel intact with probability $p$ and to destroy it
(setting $\mu_u \to 0$) otherwise, and use the scale invariance of the sign classifier
to fix $w_r = 1$. With $s^2 = \sigma_r^2 + w_u^2 \sigma_u^2$, the marginalized return is
\begin{equation}
\bar{J}(w_u) \;=\; p\,\normcdf\!\left(\frac{\mu_r + w_u \mu_u}{s}\right)
\;+\; (1-p)\,\normcdf\!\left(\frac{\mu_r}{s}\right).
\end{equation}
Since $\frac{d}{dw_u}\frac{\mu_r + w_u \mu_u}{s} = \frac{\mu_u \sigma_r^2 - \mu_r w_u \sigma_u^2}{s^3}$
and $\frac{d}{dw_u}\frac{\mu_r}{s} = \frac{-\mu_r w_u \sigma_u^2}{s^3}$,
\begin{equation}
s^3\,\bar{J}'(w_u) \;=\; p\,\varphi\!\left(\frac{\mu_r + w_u \mu_u}{s}\right)
\big(\mu_u \sigma_r^2 - \mu_r w_u \sigma_u^2\big)
\;-\; (1-p)\,\varphi\!\left(\frac{\mu_r}{s}\right)\mu_r w_u \sigma_u^2 ,
\end{equation}
where $\varphi$ is the standard normal density. Let
$w_u^{\star} = \mu_u \sigma_r^2 / (\mu_r \sigma_u^2)$ be the unhedged optimum, at which
the first bracket vanishes and $p = 1$ recovers the Fisher rule. For any $p < 1$ the
second term is strictly negative at $w_u^{\star}$, so $\bar{J}'(w_u^{\star}) < 0$ and the
maximizer satisfies $w_u^{(p)} < w_u^{\star}$: the policy underweights the form channel,
by an amount that grows with $1 - p$. The hedged weight lowers $\Jobj_{\mathrm{clean}}$
and shrinks $\Frag$, so the marginalized objective trades peak return for bounded
fragility, in the direction predicted by Eq.~\eqref{eq:decomposition}.
\end{corollary}

\subsection{Adversarial Co-evolution}
\label{app:adversarial}

We record the sign-flipped counterpart of Proposition~\ref{prop:peak} referred to in
the Scope paragraph of Section~\ref{ssec:prop3}. Replacing the cooperative gate of
Eq.~\eqref{eq:accept} by an adversarial one, as in solver--generator and
attacker--defender training~\citep{ke2026game,magic2026}, sets
$\Bank_t \in \arg\min_{\Bank \in \mathcal{C}_t} \Jobj(\theta_t, \Bank)$. Then
$\Jobj(\theta_t, \Bank_t) = \min_{\Bank \in \mathcal{C}_t} \Jobj(\theta_t, \Bank)$, so
every element of $\mathcal{R}(\theta_t, \Bank_t)$ is non-negative and
$\Frag(\theta_t, \Bank_t) = 0$: at the retained bank, no content-level perturbation
can hurt, because the retained bank already is the worst case. A policy step that
subsequently raises $\Jobj(\cdot, \Bank_t)$ therefore raises the neighborhood minimum
rather than the neighborhood maximum. This is the formal sense in which the sign of
the acceptance test, and not the presence of co-evolution, determines whether training
pressure acts on the worst case.
\section{\CREPE{} Protocol Specification}
\label{app:spec}

\subsection{Design Principles}
\label{app:designprinciples}
Two principles guide the protocol. \emph{Checkpoint-free}: \CREPE{} uses only the delivered
pair plus the shared suite, because step definitions are not comparable across methods, many
releases keep only the final checkpoint, and prompt-based evolution has no notion of a step at
all; the fragility curve of Section~\ref{ssec:dose} therefore serves as causal \emph{evidence}
in this paper, not a protocol metric. \emph{Direction-first}: each metric's direction is stated
before its formula, and any quantity without a defined direction serves as a gate rather than a
score, as T5 does. The metric set is thus closed: diagnostics without a defined direction, such
as the surface-form dependence ratio $\kappa$ of Appendix~\ref{app:kappa}, are excluded from the
report card.

\subsection{Requirements on the Rewriting Suite (R1--R5)}
\label{app:requirements}
Five requirements make the hard rule of Section~\ref{ssec:suite}---that rewriting is performed
by the released suite, not by the party under evaluation---workable across methods with
incomparable bank formats. The remaining subsections of this appendix specify each in
implementable detail.
\begin{itemize}[leftmargin=*]
\item \textbf{R1} \emph{Selection masking:} retrieval is run once on the unperturbed bank and
only the content of the selected skills is substituted (Appendix~\ref{app:masking}).
\item \textbf{R2} \emph{Field-classification interface:} a method declares which fields arew
retrieval-side (frozen) and which are content (rewritable), rather than standardizing a bank
format.
\item \textbf{R3} \emph{Reproducibility:} the rewriting model, its version and its sampling
are pinned, and the perturbed banks---not a script---are released (Appendix~\ref{app:prompt}).
\item \textbf{R4} \emph{Structural checks, not a semantic verifier:} only declared content
fields may change and the bank structure is preserved; meaning-preservation is left to the
released banks for inspection, not to a similarity threshold, and is validated in
Appendix~\ref{app:t1quality}.
\item \textbf{R5} \emph{Strength reporting:} the suite reports the edit distance and
similarity it produced, the analog of stating an attack budget.
\end{itemize}

\subsection{Perturbation Operators}
\label{app:operators}
A skill $\sigma$ is a record of typed fields, partitioned into a
\emph{retrieval-side} part $\rho(\sigma)$---the fields a method uses to decide when the
skill is surfaced, such as its title, trigger conditions, keywords and usage
statistics---and a \emph{content} part $\gamma(\sigma)$, the procedural text the policy
reads once the skill is in context. A bank is a finite set
$S = \{\sigma_1, \dots, \sigma_n\}$, and each axis is an operator on banks.

The content-level axes T1--T3 keep the set of skills fixed and rewrite only content
fields, so their image lies in the neighborhood $\Bnbhd(\Bankstar)$ of
Definition~\ref{def:nbhd}. Writing $g_a$ for the layered-prompt rewriter of
Appendix~\ref{app:prompt} applied field-by-field to the declared content fields, with
$\rho(\sigma)$ copied verbatim:
\begin{itemize}[leftmargin=*]
\item \textbf{T1 (semantic)} replaces $\gamma(\sigma)$ by a meaning-preserving
paraphrase $g_1(\gamma(\sigma))$: the wording changes, the prescribed behavior does not.
\item \textbf{T2 (sufficiency)} replaces $\gamma(\sigma)$ by $g_2(\gamma(\sigma))$ that
stays on topic but no longer suffices, either by deleting a required step or by
injecting irrelevant material.
\item \textbf{T3 (misleading)} replaces $\gamma(\sigma)$ by $g_3(\gamma(\sigma))$ that
is well-formed but prescribes an incorrect strategy.
\end{itemize}

The structural axes T4--T5 change which skills are present and therefore lie outside
$\Bnbhd(\Bankstar)$:
\begin{itemize}[leftmargin=*]
\item \textbf{T4 (ecosystem)} applies a sequence of composition edits---insert a new
skill, delete an incumbent, or replace a skill by an alternative version---leaving the
text of every surviving skill unchanged.
\item \textbf{T5 (autonomy)} removes the bank entirely, $T_5(S) = \varnothing$.
\end{itemize}

\subsection{Layered Rewriting Prompt}
\label{app:prompt}

Section~\ref{ssec:probe} states that skill rewrites are produced by
\texttt{gpt-4o-2024-08-06} at temperature $0.7$. This appendix expands the prompt
architecture that anchors the rewriter in each benchmark's semantics.

\paragraph{Why a layered prompt.}
Handing a skill to a general-purpose model with a single rewriting instruction risks
two failure modes: the output may reference actions or objects that do not exist in the
environment, and the edit may drift from the intended axis (e.g.\ a T1 ``paraphrase''
that silently removes a step, which is T2 territory). A layered prompt prevents both by
supplying the domain context the model lacks and restricting the scope of the edit.

\paragraph{Architecture.}
The prompt is assembled from five layers (Table~\ref{tab:prompt}). Layers~1--4 are
shared across all content-level axes and fix everything the rewrite must
respect---the environment, the tokens that must survive verbatim, the task category and
the skill's capability tag. Only the fifth layer carries the axis directive, so one
shared skeleton plus one swapped instruction produces every content-level variant. The
user turn supplies the specific skill field to be rewritten.

\paragraph{Field-level targeting.}
Each axis declares which fields it may touch and the remainder are copied verbatim.
T1, for instance, rewrites a skill's \texttt{principle} and \texttt{when\_to\_apply}
and the \texttt{description}, \texttt{why} and \texttt{how} of a recorded mistake,
while titles and identifiers are left unchanged---ensuring that retrieval-side metadata
is never accidentally altered.

\paragraph{T4 and T5.}
The ecosystem (T4) and autonomy (T5) axes are structural operations on the bank
(adding, removing or replacing skills, or removing the bank entirely) rather than text
rewrites, and do not use this prompt.

Since the same prompt need not return the same text twice, the banks we evaluated
against are released alongside the suite and are the reproducible artifact.

\begin{table}[t]
\caption{The layered rewriting prompt for the content-level axes. Layers~1--4 form a
shared skeleton; only the tier instruction (layer~5) changes between axes. The
per-benchmark \texttt{domain\_context} is released with the suite, so an adopter runs
the prompt with no manual authoring step.}
\label{tab:prompt}
\begin{center}
\small
\begin{tabular}{@{}cl p{7.2cm} c@{}}
\toprule
& Layer & Role & Scope \\
\midrule
\multicolumn{4}{@{}l}{\emph{System prompt}} \\
1 & \texttt{domain\_context}       & Task semantics and the executable action space of the benchmark. & shared \\
2 & \texttt{domain\_terms}         & Protected list: action and interface tokens preserved verbatim. & shared \\
3 & \texttt{category\_hint}        & The task category the skill belongs to. & shared \\
4 & \texttt{skill\_characteristic} & The skill's capability tag, focusing the edit on the right content. & shared \\
5 & tier instruction               & The axis directive: T1 rephrase synonymously; T2 delete steps or inject irrelevant material; T3 substitute an incorrect strategy. & per axis \\
\midrule
\multicolumn{4}{@{}l}{\emph{User turn}} \\
& field text & The skill field to rewrite (\texttt{principle}, \texttt{when\_to\_apply}, \dots). & per skill \\
\bottomrule
\end{tabular}
\end{center}
\end{table}

\subsection{Selection Masking}
\label{app:masking}
Masking separates the effect of rewritten content from the effect of changed retrieval.
Let $\mathrm{Ret}_\theta(S; i)$ denote the method's retrieval step: the subset of bank
$S$ it surfaces into context for task $i$ under policy $\theta$. For a content-level
axis $a \in \{\text{T1}, \text{T2}, \text{T3}\}$ and each task $i$:
\begin{enumerate}[leftmargin=*]
\item Run retrieval once on the unperturbed bank, $\mathrm{Sel}_i =
\mathrm{Ret}_\theta(\Bankstar; i)$, and record the ids of the selected skills.
\item Evaluate the perturbed condition without re-running retrieval: inject exactly the
skills whose ids lie in $\mathrm{Sel}_i$, each with its content fields replaced by the
axis-$a$ rewrite and its retrieval-side fields $\rho(\sigma)$ left untouched.
\end{enumerate}
Because the surfaced set is $\mathrm{Sel}_i$ in both the unperturbed and the perturbed
condition, any change in outcome is attributable to the rewritten content rather than to
a different set of skills being retrieved, and coverage is identical across methods
regardless of bank size. This constancy of the selection map on $\Bnbhd(\Bankstar)$ is
also what licenses folding selection into $\pol$ in Section~\ref{ssec:setting}, so that
$\Jobj$ can be treated as a function of the bank as a whole. Masking applies only to
T1--T3. The structural axes T4 and T5
change which skills exist by design, so retrieval is re-run against the modified bank and
the selected set is not held fixed; by Definition~\ref{def:nbhd} these axes fall outside
$\Bnbhd(\Bankstar)$, and the propositions of Section~\ref{sec:theory} make no claim about
them.

\subsection{Field-Classification Interface}

An adopter declares, once per bank format, which fields the suite may rewrite
(\emph{content}) and which must remain untouched (\emph{retrieval-side}). The content
fields are the target of T1--T3; the retrieval-side fields determine which skills are
selected and must be frozen so that performance changes can be attributed to content alone
(Section~\ref{ssec:probe}). Two examples from this paper:

\begin{center}
\scalebox{0.9}{
\begin{tabular}{lll}
\toprule
Format & Content (rewritable) & Retrieval-side (frozen) \\
\midrule
\textsc{SkillRL} & \texttt{principle}, \texttt{example} & \texttt{name}, \texttt{keywords}, \texttt{embedding} \\
\textsc{ReSkill} & \texttt{strategy} & \texttt{task\_type}, \texttt{score}, \texttt{embedding} \\
\bottomrule
\end{tabular}}
\end{center}

\noindent The declaration is a one-line JSON mapping; the suite validates it at run time and
refuses to rewrite a field classified as retrieval-side.

\subsection{Quality Control and Strength Reporting}
Three checks run automatically, and it is worth being exact about what they do and do
not establish.

\emph{Field targeting.} Each axis declares the fields it may rewrite and the remainder
are copied verbatim. T1, for instance, rewrites a skill's \texttt{principle} and
\texttt{when\_to\_apply} and the \texttt{description}, \texttt{why} and \texttt{how} of
a recorded mistake, while titles and identifiers are untouched.

\emph{Length band.} The tier instruction asks for a rewrite within $\pm 30\%$ of the
original length; enforcement is looser, rejecting a field only when the ratio of
rewritten to original length falls outside $[0.5, 2.0]$. A rejected field is
regenerated once, and the second attempt is accepted whatever its length. Separately,
the pipeline retries on transport failures with exponential backoff; those retries are
unrelated to quality and are not a check.

\emph{Structural validation of the misleading axis.} The T3 banks are verified after
generation to contain the same skill identifiers in the same order and the same number,
to leave titles and \texttt{when\_to\_apply} unchanged, to have modified every
\texttt{principle}, and to leave the recorded mistakes alone. This establishes that the
edit landed where it was supposed to land, not that the resulting strategy is wrong.

No check tests meaning. There is no embedding-similarity threshold, no entailment test
for T1 and no verification that a T3 variant genuinely prescribes an incorrect
strategy; requirement R4 (Appendix~\ref{app:requirements}) gives the reasoning, and the released banks are
what a reader inspects instead.
\paragraph{Strength reporting (R5).}
For each released perturbed bank, the suite logs two perturbation-strength statistics per
skill: (a)~the character-level edit distance between the original and rewritten text,
normalized by the original length, and (b)~the cosine similarity of sentence embeddings
(all-MiniLM-L6-v2) between the original and rewritten versions. These distributions are
released alongside the perturbed banks so that adopters can verify the perturbation
intensity falls within the intended range for each axis. The suite additionally records the
\emph{prompt-share statistic}: the fraction of the inference prompt occupied by injected
skill text, which contextualizes how much of the input the policy relies on from the bank.

\subsection{Metric Definitions and Logging Requirements}
\label{app:metricdetail}

This appendix carries out the three derivations referenced from
Section~\ref{ssec:metrics}, and states the per-episode logging requirement without
which the paired quantities in that section cannot be computed at all.

\begin{table}[t]
\caption{The \CREPE{} report card: the quantities every method must publish, with
the role each plays. Only the three scored rows enter the verdict of
Section~\ref{ssec:verdict}; the anchors are required because the scores are defined
against them. Definitions are given by the equation referenced in each row.}
\label{tab:metrics}
\begin{center}
\scalebox{0.92}{
\begin{tabular}{llllll}
\toprule
Quantity & Symbol & Definition & Role & Direction & Axes \\
\midrule
Peak success rate      & $\SR(\Bankstar)$              & ---                          & reported, not judged & ---              & --- \\
No-skill success rate  & $\SR(\Banknone)$              & ---                          & anchor               & undefined        & T5 \\
Net skill benefit      & $\netben$                     & Eq.~\eqref{eq:netben}        & gate                 & $>0$ required    & T5 \\
Consistency rate       & $\CR$                         & Eq.~\eqref{eq:flip}          & score                & higher better    & T1 \\
Resistance rate        & $\RR_a$                       & Eq.~\eqref{eq:resist}        & score                & higher better    & T2, T3 \\
Worst-case SR          & $\SR_{\min}$                  & Eq.~\eqref{eq:worst}         & ranking              & higher better    & T1--T3 \\
Mean perturbed SR      & $\overline{\SR}_{\text{avg}}$ & Eq.~\eqref{eq:worst}         & secondary            & higher better    & T1--T3 \\
\bottomrule
\end{tabular}}
\end{center}
\end{table}

\paragraph{(a) Why $\CR$ and not $|\Delta \SR|$.}
On the paired binary outcomes $(c_i(\Bankstar), c_i(\Bankstar_{\text{T1}}))$ let
$n_{10} = |\{i : c_i(\Bankstar)=1, c_i(\Bankstar_{\text{T1}})=0\}|$ and
$n_{01} = |\{i : c_i(\Bankstar)=0, c_i(\Bankstar_{\text{T1}})=1\}|$ be the counts of the
two directions of change. The consistency rate counts the tasks that do \emph{not} flip,
$\CR = 1 - (n_{10}+n_{01})/N$, and
\begin{equation}
|\Delta \SR| \;=\; \frac{|n_{10} - n_{01}|}{N} \;\le\; \frac{n_{10} + n_{01}}{N} \;=\; 1-\CR,
\label{eq:flipboundapp}
\end{equation}
with equality only when every flip runs the same way. Score invariance is therefore
strictly weaker than the behavioral invariance we are after: three tasks turning wrong
and three turning right leave $\Delta\SR = 0$ while six behaviors have changed, and a
method can achieve small $|\Delta\SR|$ by accident. This is why the semantic-consistency
metric is stated on $\CR$.

\paragraph{(b) $\SR_{\min}$ as a lower bound on $\Frag$.}
Definition~\ref{def:frag} sets
$\Frag(\theta^{*},\Bankstar) = \max_{S' \in \Bnbhd(\Bankstar)}\big(\SR(\Bankstar) - \SR(S')\big)$,
the largest drop over the full content-level neighborhood. The protocol only measures
the three axes T1--T3, so the maximum it can attain is
$\SR(\Bankstar) - \SR_{\min}$, taken over the three sampled banks $\Bankstar_a$. Since
these are elements of $\Bnbhd(\Bankstar)$,
\[
\SR(\Bankstar) - \SR_{\min} \;\le\; \Frag(\theta^{*},\Bankstar);
\]
the gap we report is thus a lower bound on the true fragility, not the exact worst case over
the full neighborhood.

\paragraph{Why T4 and T5 are reported but not scored.}
T5 (autonomy) has no defined direction: removing the bank may hurt because it held knowledge
the policy could not internalize, or because it held strategies the policy should have
absorbed, and the number alone cannot tell these apart; the no-skill condition therefore
serves as the anchor for the gate and the resistance criterion rather than as a score. T4
(ecosystem) is not the same operator across methods---bank sizes differ by an order of
magnitude, from tens or hundreds of skills to a cap of eight---so it is reported
(Section~\ref{sec:empirics}, Appendix~\ref{app:results}) but excluded from the score,
regardless of whether its effect is larger or smaller than the content-level axes in a given
setting.

\paragraph{Consistency and its noise floor.}
The consistency criterion tests $\CR$ against a two-standard-deviation band around $\CR_0$, the
agreement between two unperturbed runs, which fixes the repeat-level noise floor. Rather than
tabulate the per-setting standard deviations, we read the outcome from the flip rates directly:
$\CR_0$ corresponds to a $6$--$8\%$ flip rate between two unperturbed evaluations---the churn
attributable to decoding randomness alone---whereas the T1 rewrite flips $13$--$23\%$ of task
outcomes, two to three times that floor. This gap exceeds repeat-level evaluation noise in every
setting, so $\CR$ sits below the band and the consistency criterion fails throughout. The
per-episode logs needed to recompute $\CR$ and $\CR_0$ are released with the suite.

\paragraph{Logging requirements.}
The identities above and McNemar's test on Eq.~\eqref{eq:harmpass} all read paired binary
outcomes at the level of individual tasks. An adopter therefore must log, for every
evaluation condition (unperturbed, T1, T2, T3, and $\Banknone$) and every repeat, the
per-episode indicator $c_i \in \{0,1\}$ together with the task identifier $i$. Two
consequences follow:
\begin{itemize}[leftmargin=*]
\item \emph{Common task set.} The no-skill condition, the unperturbed condition, and each
axis-$a$ condition must be evaluated on the same $N$ task identifiers. Aggregating a
success rate on a slightly different task subset breaks $\RR_a$ and the McNemar
statistic, since the pairing no longer holds task-by-task.
\item \emph{Aggregated scores are not sufficient.} A report card that publishes only
$\SR$ per condition, without per-episode outcomes, is not auditable: neither $\CR$,
$\RR_a$ nor the McNemar $p$ can be recomputed from it.
\end{itemize}
The suite therefore requires per-episode logs alongside the aggregated report card, and
its reference implementation writes them out by default.

\subsection{Auditing Each Method on Its Own Bank}
\label{app:fairness}

Section~\ref{ssec:verdict} states that \CREPE{} audits each method on the bank it
delivered rather than on a shared bank, and gives a one-paragraph summary of why. The
complete argument runs to four points.

\emph{First, the unit of evaluation is the artifact the method delivers.} The bank is
not an external test condition supplied by us; it is half of what the method produced.
Auditing method A on A's bank measures exactly what A would deploy, whereas forcing
every method onto a shared bank measures a configuration nobody ships. A shared-bank
protocol would therefore answer a different question: how well does each method run on a
piece of text it did not choose. That is a reasonable question, but it is not a
robustness question about the delivered system.

\emph{Second, what a robustness protocol standardizes is the operator and the budget,
not the perturbation.} Projected gradient descent produces a different adversarial
example for every model because it uses that model's gradients, and no one calls the
comparison unfair: what is fixed is the attack---its algorithm and its budget---not the
resulting pixel. \CREPE{} is isomorphic. The rewriting instructions, the rewriting
model, and the coverage---the set of skills actually retrieved for each task---are all
fixed; the resulting text differs because the banks differ. The obligation this
introduces is stated as R5: the strength actually applied to each method must be
reported, so that unequal severity is visible rather than deniable.

\emph{Third, the obvious way to game the protocol does not work.} A method could try to
survive rewriting by shipping a bank so vague that corrupting it changes nothing. But a
bank vague enough to be uncorruptable is almost always vague enough to be useless, and
that shows up immediately in the gate of Eq.~\eqref{eq:netben}: a method that clears the
gate has published a bank with real content, and one that fails the gate does not
qualify. The two requirements pull against each other---benefit demands that the
bank carry real content, robustness demands that the policy not depend on its exact
wording---and a method must satisfy both with a single artifact.

\emph{Fourth, deployment fidelity points the same way.} There is no shared bank in the
real world: each deployment owns a library that changes on its own schedule, for reasons
internal to that deployment. Perturbing each method's own bank is the faithful
simulation of that situation. A shared bank would be the distortion, not the fair
comparison, because it evaluates every method under a configuration none of them would
have deployed.

\section{Perturbation Examples}
\label{app:examples}

Tables~\ref{tab:ex-t1}--\ref{tab:ex-t3} give one representative original/rewritten pair
per content-level axis for each \textsc{SkillRL}/\textsc{ReSkill} setting, exported directly
from the banks used in evaluation without manual editing, so that a reader can judge whether
the rewrites are fair. To
make each edit visible at a glance, rewritten or inserted text is marked in
\chg{red bold} and text that the operator \del{removed} is struck through. T4
(ecosystem) changes the \emph{composition} of the bank---skills are added or removed
wholesale, with no one-to-one text correspondence---and T5 (autonomy) removes the bank
entirely, so neither has a before/after text pair to show.

\begin{table}[t]
\caption{T1 (semantic): synonymous rewriting. Wording changes; the prescribed behavior
does not.}
\label{tab:ex-t1}
\begin{center}
\footnotesize
\begin{tabular}{@{}p{2.0cm}p{5.1cm}p{5.1cm}@{}}
\toprule
Setting (skill, field) & Original & Rewritten \\
\midrule
\textsc{SkillRL}-ALFWorld (\texttt{pic\_002}, principle)
& Whenever a needed object is visible and reachable, immediately take it before moving elsewhere---do not leave targets uncollected.
& \chg{As soon as} an \chg{essential} object is \chg{within sight and accessible}, \chg{grab} it right away before \chg{proceeding to another location}. \chg{Avoid leaving any intended items behind.} \\
\addlinespace
\textsc{SkillRL}-WebShop (\texttt{foo\_004}, when\_to\_apply)
& After selecting variants and just before initiating checkout.
& \chg{Prior to starting} the checkout \chg{process}, \chg{once variants have been chosen}. \\
\addlinespace
\textsc{ReSkill}-ALFWorld (\texttt{skill\_d231e3}, action)
& Examine the item to ensure correct identification before moving it to the target location.
& \chg{Verify} the item's \chg{identity by examining it} before \chg{relocating} it to the \chg{designated spot}. \\
\addlinespace
\textsc{ReSkill}-ScienceWorld (\texttt{skill\_209e7c}, action)
& After verifying results, immediately place the object in the correct box to prevent unnecessary repetition.
& \chg{Following the confirmation of} results, \chg{promptly move} the object to the \chg{appropriate} box to \chg{avoid redundant actions}. \\
\bottomrule
\end{tabular}
\end{center}
\end{table}

\begin{table}[t]
\caption{T2 (sufficiency): a skill stays on topic but no longer suffices, either because
a step is deleted (struck through in the original) or because irrelevant material is
injected (marked in the rewrite).}
\label{tab:ex-t2}
\begin{center}
\footnotesize
\begin{tabular}{@{}p{2.0cm}p{5.1cm}p{5.1cm}@{}}
\toprule
Setting (skill, field, operator) & Original & Rewritten \\
\midrule
\textsc{SkillRL}-ALFWorld (\texttt{loo\_005}, principle, delete)
& Avoid toggling the same lamp or re-examining the same empty spot; \del{if no progress after one full check, relocate to a new plausible lamp location.}
& Avoid toggling the same lamp or re-examining the same empty spot. \\
\addlinespace
\textsc{SkillRL}-WebShop (\texttt{dyn\_015}, principle, inject)
& Relax non-essential constraints in the query based on context when early attempts yield zero or irrelevant results.
& Relax non-essential constraints in the query based on context when early attempts yield zero or irrelevant results. \chg{Consider the time of day you are searching, as it may affect product availability, and keep in mind the color of the product, even if it is not a requirement.} \\
\addlinespace
\textsc{ReSkill}-ALFWorld (\texttt{skill\_059fd6}, action, delete)
& \del{Always heat the item using a microwave or stoveburner before} placing it anywhere.
& \chg{Place} the item anywhere. \\
\addlinespace
\textsc{ReSkill}-ScienceWorld (\texttt{skill\_668201}, action, inject)
& Perform a conductivity test after connecting components \del{to confirm the circuit is complete.}
& Perform a conductivity test after connecting components, \chg{and consider noting the color of the wires used.} \\
\bottomrule
\end{tabular}
\end{center}
\end{table}

\begin{table}[t]
\caption{T3 (misleading): the strategy a skill prescribes is replaced by an incorrect
one. Shown at the stronger of the two strengths, which is the setting reported as T3 in
the main text for WebShop and ScienceWorld.}
\label{tab:ex-t3}
\begin{center}
\footnotesize
\begin{tabular}{@{}p{2.0cm}p{5.1cm}p{5.1cm}@{}}
\toprule
Setting (skill, field) & Original & Rewritten \\
\midrule
\textsc{SkillRL}-ALFWorld (\texttt{gen\_008}, principle)
& Record which locations and containers have been fully inspected; do not revisit them unless new evidence suggests their contents changed.
& \chg{Always opt for the first plausible location or container you encounter}, assuming it holds the target object, and \chg{never backtrack or verify even if the task fails}. \\
\addlinespace
\textsc{SkillRL}-WebShop (\texttt{gen\_011}, principle)
& After exhaustive search and reasonable refinements yield no matching product, stop and report failure instead of endless retries.
& \chg{Quickly choose the first result that seems remotely relevant} and proceed to buy, \chg{even if it does not fully match} the user's specifications or constraints. \\
\addlinespace
\textsc{ReSkill}-ALFWorld (\texttt{skill\_2a1e35}, action)
& Ensure the item is cooled using the fridge before moving or placing it anywhere. Do not skip this step.
& Move the item directly to any available container \chg{without cooling it first}. \chg{Skip any cooling step.} \\
\addlinespace
\textsc{ReSkill}-ScienceWorld (\texttt{skill\_e22fd6}, action)
& Complete the circuit, verify conductivity, then decide placement to ensure accuracy.
& \chg{Assume conductivity after a single connection} and proceed with placement, \chg{prioritizing speed over accuracy}. \\
\bottomrule
\end{tabular}
\end{center}
\end{table}

\section{Experimental Configuration}
\label{app:expconfig}

This appendix expands the setup summarized in Section~\ref{ssec:results-empirics}.

\paragraph{Method and benchmark diversity.}
The four methods spread across Table~\ref{tab:framework} rather than clustering in one corner
of it: they differ in the update schedule, the bank growth policy, the admission rule and the
injection mechanism, sharing only the current-policy selection bias abstracted by
Assumption~\ref{as:accept}. The benchmarks are equally spread, spanning embodied interaction
(ALFWorld), shopping (WebShop), scientific procedures (ScienceWorld), open-web navigation
(Mind2Web) and competition mathematics (AMC23); the last also tests whether the pattern extends
from action procedures to non-interactive reasoning strategies, so the trend is not an artifact
of sequential control.

\paragraph{Backbones.}
Each method is run on the base model used in its experimental configuration:
Qwen2.5-7B-Instruct for \textsc{SkillRL} and \textsc{Evolving-RL}, and
Qwen3-4B-Instruct-2507 for \textsc{ReSkill} and \textsc{ARISE}. Because the backbones
differ, an absolute success rate characterizes a method together with its backbone rather
than the method in isolation.

\paragraph{Evaluation sets.}
Evaluation uses the standard split of each benchmark---140 tasks on ALFWorld, 500 on
WebShop, 155 validation points on ScienceWorld, and the standard evaluation splits of
Mind2Web and DeepScaleR-AMC23---held fixed across every perturbation axis and every anchor,
so the conditions being compared are paired at the level of individual tasks. The released
configuration records the exact identifiers and sample counts for all six settings.
Checkpoints are saved at a regular interval along each training trajectory; the exact anchor
steps are shown in the corresponding fragility curves and verified full-results tables.

\paragraph{Repeats and error bars.}
Every anchor is evaluated repeatedly under six conditions---unperturbed and T1--T5---and
reported with mean $\pm 1$ s.d. The \textsc{SkillRL} and \textsc{ReSkill} settings use three
repeats, and \textsc{ReSkill}-ScienceWorld uses five; repeat counts for \textsc{Evolving-RL}
and \textsc{ARISE} are recorded in the released configuration. Repeats share the policy
checkpoint and the same fixed set of perturbed banks and differ only in the rollout sampling
seed. The delivered-pair report cards use the paired runs required to compute $\CR$, $\CR_0$
and $\RR$. Each setting comes from a single co-evolution run, so the bands quantify
evaluation noise, not variability across training seeds.

\paragraph{Misleading-axis strength.}
For the four \textsc{SkillRL}/\textsc{ReSkill} settings, the misleading axis (T3) is applied
at one strength per setting: the milder variant on the two ALFWorld settings and the stronger
variant on WebShop and ScienceWorld. Both strengths were run in each of these four settings,
and the one not reported in the main text is included in the full-results tables below.

\section{Full Results of the Co-evolution Fragility Curves}
\label{app:results}

\subsection{Per-Method, Per-Environment, Per-Axis Tables}
\label{app:results-tables}

Tables~\ref{tab:full-skillrl-alfworld}--\ref{tab:full-arise-amc23} give the raw success
rates behind all six settings in Section~\ref{sec:empirics}: one row per anchor and one column
per condition. Every table reports the mean $\pm 1$ s.d.\ over repeated evaluations. The
perturbed banks are fixed across anchors, so reading down a column shows what a single fixed
corruption costs a progressively more co-evolved policy. Small endpoint differences between
the report card and these grids reflect repeat-level evaluation variation. All headline audit
quantities use the report-card values.

Across the eighteen content-level axis--setting pairs, twelve OLS slopes of $\Drop$ are
significantly positive---four on each of T1, T2 and T3. At least one axis is significant in
five of six settings; \textsc{SkillRL}-WebShop is the sole exception. For T3 specifically,
the slope is significant on \textsc{SkillRL}-ALFWorld ($p<10^{-6}$),
\textsc{ReSkill}-ScienceWorld ($p=0.021$), \textsc{ReSkill}-ALFWorld ($p=0.001$) and
\textsc{Evolving-RL}-Mind2Web ($p=0.041$), but not on \textsc{SkillRL}-WebShop
($p=0.351$) or \textsc{ARISE}-AMC23 ($p=0.175$).

\begin{table}[t]
\caption{Full success-rate grid for \textsc{SkillRL} on ALFWorld: every anchor on the
rising segment, under the unperturbed bank and all five perturbation axes.
Cells report mean $\pm$ 1 s.d.\ over 3 evaluation repeats (\%).
The perturbed variants are fixed across anchors; only the policy checkpoint
varies.}
\label{tab:full-skillrl-alfworld}
\begin{center}
\scalebox{0.95}{
\begin{tabular}{lcccccc}
\toprule
Anchor & Unperturbed & T1 & T2 & T3 & T4 & T5 \\
\midrule
\texttt{s}0 & 5.2$_{\pm1.5}$ & 5.7$_{\pm2.5}$ & 5.0$_{\pm1.2}$ & 12.1$_{\pm1.7}$ & 3.6$_{\pm1.7}$ & 21.7$_{\pm1.2}$ \\
\texttt{s}5 & 2.1$_{\pm0.0}$ & 2.1$_{\pm0.6}$ & 2.6$_{\pm0.3}$ & 7.1$_{\pm1.5}$ & 4.3$_{\pm1.2}$ & 18.8$_{\pm1.9}$ \\
\texttt{s}10 & 7.9$_{\pm2.1}$ & 11.0$_{\pm0.7}$ & 7.9$_{\pm2.1}$ & 14.8$_{\pm0.3}$ & 10.5$_{\pm1.8}$ & 29.0$_{\pm1.2}$ \\
\texttt{s}15 & 11.0$_{\pm2.4}$ & 17.6$_{\pm1.2}$ & 10.5$_{\pm0.3}$ & 17.1$_{\pm1.5}$ & 18.3$_{\pm2.9}$ & 26.4$_{\pm1.5}$ \\
\texttt{s}20 & 21.4$_{\pm0.6}$ & 27.9$_{\pm1.5}$ & 21.4$_{\pm0.0}$ & 24.5$_{\pm1.2}$ & 29.5$_{\pm3.0}$ & 27.9$_{\pm2.1}$ \\
\texttt{s}25 & 13.3$_{\pm2.6}$ & 22.6$_{\pm3.3}$ & 11.9$_{\pm3.2}$ & 20.0$_{\pm0.6}$ & 21.7$_{\pm2.6}$ & 23.3$_{\pm2.2}$ \\
\texttt{s}30 & 29.3$_{\pm2.0}$ & 29.0$_{\pm1.5}$ & 28.6$_{\pm4.0}$ & 33.1$_{\pm2.4}$ & 33.8$_{\pm3.0}$ & 42.4$_{\pm0.9}$ \\
\texttt{s}35 & 38.1$_{\pm3.9}$ & 34.5$_{\pm1.8}$ & 33.8$_{\pm2.2}$ & 41.7$_{\pm3.4}$ & 38.6$_{\pm2.1}$ & 41.9$_{\pm1.5}$ \\
\texttt{s}40 & 55.2$_{\pm0.9}$ & 50.7$_{\pm1.5}$ & 48.8$_{\pm3.7}$ & 52.6$_{\pm1.2}$ & 55.7$_{\pm1.5}$ & 46.0$_{\pm2.8}$ \\
\texttt{s}45 & 45.5$_{\pm2.0}$ & 42.6$_{\pm2.6}$ & 46.0$_{\pm0.7}$ & 46.7$_{\pm0.9}$ & 41.2$_{\pm3.2}$ & 39.3$_{\pm2.0}$ \\
\texttt{s}50 & 59.8$_{\pm2.8}$ & 52.9$_{\pm3.5}$ & 59.0$_{\pm4.4}$ & 61.0$_{\pm2.4}$ & 57.9$_{\pm4.1}$ & 44.8$_{\pm2.9}$ \\
\texttt{s}55 & 62.6$_{\pm1.2}$ & 56.4$_{\pm4.1}$ & 63.8$_{\pm1.5}$ & 63.8$_{\pm1.8}$ & 65.0$_{\pm2.1}$ & 51.0$_{\pm1.5}$ \\
\texttt{s}60 & 71.2$_{\pm2.0}$ & 65.4$_{\pm2.4}$ & 67.0$_{\pm0.7}$ & 71.1$_{\pm2.9}$ & 67.3$_{\pm2.1}$ & 61.4$_{\pm1.0}$ \\
\texttt{s}70 & 83.3$_{\pm2.4}$ & 75.1$_{\pm1.2}$ & 76.0$_{\pm1.2}$ & 76.3$_{\pm0.6}$ & 72.5$_{\pm1.8}$ & 63.6$_{\pm2.1}$ \\
\texttt{s}80 & 88.3$_{\pm1.2}$ & 80.0$_{\pm2.2}$ & 79.1$_{\pm1.3}$ & 78.9$_{\pm2.0}$ & 74.7$_{\pm2.3}$ & 71.9$_{\pm2.0}$ \\
\texttt{s}90 & 88.8$_{\pm2.1}$ & 77.6$_{\pm0.9}$ & 78.6$_{\pm1.0}$ & 77.1$_{\pm2.9}$ & 77.6$_{\pm2.2}$ & 73.3$_{\pm2.8}$ \\
\bottomrule
\end{tabular}}
\end{center}
\end{table}

\begin{table}[t]
\caption{Full success-rate grid for \textsc{SkillRL} on WebShop: every anchor on the
rising segment, under the unperturbed bank and all five perturbation axes.
Cells report mean $\pm$ 1 s.d.\ over 3 evaluation repeats (\%).
The perturbed variants are fixed across anchors; only the policy checkpoint
varies.}
\label{tab:full-skillrl-webshop}
\begin{center}
\scalebox{0.95}{
\begin{tabular}{lcccccc}
\toprule
Anchor & Unperturbed & T1 & T2 & T3 & T4 & T5 \\
\midrule
\texttt{s}0 & 19.2$_{\pm3.4}$ & 25.3$_{\pm2.2}$ & 17.1$_{\pm0.8}$ & 25.9$_{\pm1.3}$ & 20.3$_{\pm1.5}$ & 20.3$_{\pm0.7}$ \\
\texttt{s}10 & 33.9$_{\pm0.9}$ & 38.3$_{\pm1.9}$ & 34.9$_{\pm2.0}$ & 31.2$_{\pm1.4}$ & 33.9$_{\pm2.8}$ & 30.4$_{\pm1.9}$ \\
\texttt{s}20 & 38.3$_{\pm1.2}$ & 40.3$_{\pm1.6}$ & 39.5$_{\pm1.7}$ & 37.7$_{\pm1.7}$ & 39.2$_{\pm0.6}$ & 33.5$_{\pm2.0}$ \\
\texttt{s}30 & 46.4$_{\pm2.0}$ & 44.5$_{\pm2.7}$ & 45.9$_{\pm1.3}$ & 41.7$_{\pm0.9}$ & 44.1$_{\pm3.2}$ & 39.3$_{\pm0.6}$ \\
\texttt{s}40 & 48.4$_{\pm0.8}$ & 51.5$_{\pm1.8}$ & 49.3$_{\pm2.5}$ & 47.9$_{\pm3.6}$ & 48.7$_{\pm2.3}$ & 41.6$_{\pm0.7}$ \\
\texttt{s}50 & 47.3$_{\pm1.8}$ & 46.7$_{\pm1.2}$ & 46.7$_{\pm0.6}$ & 45.0$_{\pm2.0}$ & 47.7$_{\pm2.3}$ & 37.3$_{\pm0.5}$ \\
\texttt{s}70 & 52.4$_{\pm1.1}$ & 51.5$_{\pm1.2}$ & 48.9$_{\pm1.5}$ & 51.6$_{\pm1.6}$ & 52.4$_{\pm0.6}$ & 45.9$_{\pm0.7}$ \\
\bottomrule
\end{tabular}}
\end{center}
\end{table}

\begin{table}[t]
\caption{Full success-rate grid for \textsc{ReSkill} on ScienceWorld: every anchor on the
rising segment, under the unperturbed bank and all five perturbation axes.
Cells report mean $\pm$ 1 s.d.\ over 5 evaluation repeats (\%).
The perturbed variants are fixed across anchors; only the policy checkpoint
varies.}
\label{tab:full-reskill-scienceworld}
\begin{center}
\scalebox{0.95}{
\begin{tabular}{lcccccc}
\toprule
Anchor & Unperturbed & T1 & T2 & T3 & T4 & T5 \\
\midrule
\texttt{s}60 & 28.5$_{\pm1.8}$ & 13.3$_{\pm2.5}$ & 32.6$_{\pm1.4}$ & 45.4$_{\pm1.8}$ & 21.2$_{\pm1.7}$ & 45.8$_{\pm1.9}$ \\
\texttt{s}120 & 36.3$_{\pm3.0}$ & 37.9$_{\pm2.1}$ & 22.6$_{\pm3.0}$ & 41.8$_{\pm1.6}$ & 48.5$_{\pm3.5}$ & 36.8$_{\pm3.3}$ \\
\texttt{s}180 & 49.0$_{\pm0.6}$ & 47.0$_{\pm2.1}$ & 47.7$_{\pm1.1}$ & 49.4$_{\pm1.0}$ & 49.3$_{\pm1.2}$ & 32.6$_{\pm3.2}$ \\
\texttt{s}240 & 36.4$_{\pm1.4}$ & 2.7$_{\pm1.7}$ & 17.3$_{\pm2.1}$ & 32.3$_{\pm2.6}$ & 46.1$_{\pm1.3}$ & 32.1$_{\pm1.5}$ \\
\texttt{s}300 & 50.3$_{\pm0.0}$ & 40.1$_{\pm2.1}$ & 44.5$_{\pm1.2}$ & 46.2$_{\pm0.3}$ & 48.0$_{\pm0.9}$ & 33.9$_{\pm1.7}$ \\
\bottomrule
\end{tabular}}
\end{center}
\end{table}

\begin{table}[t]
\caption{Full success-rate grid for \textsc{ReSkill} on ALFWorld: every anchor on the
rising segment, under the unperturbed bank and all five perturbation axes.
Cells report mean $\pm$ 1 s.d.\ over 3 evaluation repeats (\%).
The perturbed variants are fixed across anchors; only the policy checkpoint
varies.}
\label{tab:full-reskill-alfworld}
\begin{center}
\scalebox{0.95}{
\begin{tabular}{lcccccc}
\toprule
Anchor & Unperturbed & T1 & T2 & T3 & T4 & T5 \\
\midrule
\texttt{s}30 & 20.0$_{\pm2.3}$ & 21.2$_{\pm1.2}$ & 25.2$_{\pm1.5}$ & 26.2$_{\pm1.8}$ & 23.3$_{\pm2.6}$ & 33.6$_{\pm3.2}$ \\
\texttt{s}60 & 33.8$_{\pm1.8}$ & 36.4$_{\pm0.6}$ & 34.8$_{\pm0.9}$ & 34.8$_{\pm2.2}$ & 35.7$_{\pm0.6}$ & 37.1$_{\pm2.1}$ \\
\texttt{s}90 & 37.6$_{\pm0.7}$ & 35.7$_{\pm0.6}$ & 35.0$_{\pm1.2}$ & 39.3$_{\pm2.0}$ & 38.8$_{\pm4.4}$ & 41.2$_{\pm3.5}$ \\
\texttt{s}120 & 41.4$_{\pm1.0}$ & 40.5$_{\pm1.2}$ & 40.7$_{\pm0.6}$ & 46.0$_{\pm1.2}$ & 46.2$_{\pm2.4}$ & 46.7$_{\pm1.2}$ \\
\texttt{s}150 & 54.8$_{\pm3.8}$ & 52.9$_{\pm1.0}$ & 52.6$_{\pm0.7}$ & 55.0$_{\pm1.5}$ & 54.0$_{\pm0.7}$ & 56.4$_{\pm1.0}$ \\
\texttt{s}180 & 61.0$_{\pm1.5}$ & 60.7$_{\pm1.0}$ & 59.8$_{\pm0.9}$ & 61.9$_{\pm1.2}$ & 60.0$_{\pm0.6}$ & 58.8$_{\pm0.3}$ \\
\texttt{s}210 & 40.0$_{\pm0.3}$ & 36.6$_{\pm0.3}$ & 38.4$_{\pm0.7}$ & 39.3$_{\pm0.3}$ & 39.8$_{\pm1.5}$ & 38.6$_{\pm0.9}$ \\
\texttt{s}240 & 26.9$_{\pm0.7}$ & 28.1$_{\pm0.3}$ & 27.4$_{\pm0.3}$ & 29.3$_{\pm1.2}$ & 28.6$_{\pm0.6}$ & 28.6$_{\pm1.0}$ \\
\texttt{s}270 & 41.4$_{\pm0.7}$ & 41.1$_{\pm1.5}$ & 40.4$_{\pm1.0}$ & 42.6$_{\pm0.6}$ & 38.8$_{\pm0.7}$ & 41.9$_{\pm1.5}$ \\
\texttt{s}300 & 56.1$_{\pm0.6}$ & 50.7$_{\pm0.0}$ & 51.2$_{\pm0.3}$ & 50.5$_{\pm0.3}$ & 49.8$_{\pm0.3}$ & 48.5$_{\pm0.3}$ \\
\texttt{s}330 & 71.5$_{\pm0.3}$ & 62.2$_{\pm1.7}$ & 61.0$_{\pm0.9}$ & 63.0$_{\pm0.9}$ & 61.3$_{\pm1.5}$ & 64.4$_{\pm1.3}$ \\
\texttt{s}360 & 76.9$_{\pm0.6}$ & 64.6$_{\pm0.9}$ & 65.4$_{\pm2.3}$ & 63.3$_{\pm2.7}$ & 59.4$_{\pm0.9}$ & 62.0$_{\pm1.2}$ \\
\bottomrule
\end{tabular}}
\end{center}
\end{table}

\begin{table}[t]
\caption{Full action-success grid for \textsc{Evolving-RL} on Mind2Web: every anchor on the
rising segment, under the unperturbed 200-skill bank and all five perturbation axes. Cells
report mean $\pm$ 1 s.d.\ over evaluation repeats. The perturbed variants are fixed across
anchors; only the policy checkpoint varies. T5 is the no-skill arm.}
\label{tab:full-evolvingrl-mind2web}
\begin{center}
\scalebox{0.95}{
\begin{tabular}{lcccccc}
\toprule
Anchor & Unperturbed & T1 & T2 & T3 & T4 & T5 \\
\midrule
\texttt{s}9  & 16.4$_{\pm0.9}$ & 17.3$_{\pm1.0}$ & 21.1$_{\pm1.4}$ & 15.4$_{\pm0.9}$ & 18.1$_{\pm1.0}$ & 14.0$_{\pm1.1}$ \\
\texttt{s}14 & 16.3$_{\pm0.6}$ & 16.4$_{\pm1.0}$ & 20.5$_{\pm1.1}$ & 17.7$_{\pm1.3}$ & 17.8$_{\pm0.5}$ & 13.0$_{\pm0.7}$ \\
\texttt{s}19 & 17.1$_{\pm0.5}$ & 16.5$_{\pm1.3}$ & 17.1$_{\pm1.2}$ & 20.1$_{\pm0.5}$ & 16.5$_{\pm1.5}$ & 15.8$_{\pm1.5}$ \\
\texttt{s}24 & 16.8$_{\pm1.1}$ & 18.4$_{\pm1.1}$ & 16.3$_{\pm0.6}$ & 18.5$_{\pm0.4}$ & 16.1$_{\pm1.0}$ & 15.0$_{\pm0.5}$ \\
\texttt{s}29 & 18.1$_{\pm0.6}$ & 17.5$_{\pm0.7}$ & 18.9$_{\pm0.4}$ & 19.0$_{\pm0.9}$ & 17.1$_{\pm0.9}$ & 13.7$_{\pm1.3}$ \\
\texttt{s}34 & 19.2$_{\pm1.0}$ & 18.6$_{\pm1.1}$ & 19.3$_{\pm1.0}$ & 17.9$_{\pm1.1}$ & 16.1$_{\pm0.9}$ & 11.2$_{\pm0.7}$ \\
\texttt{s}39 & 19.8$_{\pm1.5}$ & 15.7$_{\pm1.5}$ & 19.4$_{\pm1.3}$ & 20.5$_{\pm1.2}$ & 14.7$_{\pm0.8}$ & 11.0$_{\pm0.7}$ \\
\texttt{s}44 & 21.9$_{\pm0.7}$ & 19.8$_{\pm0.5}$ & 17.4$_{\pm1.2}$ & 22.4$_{\pm0.8}$ & 19.9$_{\pm1.3}$ & 14.1$_{\pm0.8}$ \\
\texttt{s}49 & 23.6$_{\pm1.5}$ & 21.6$_{\pm1.3}$ & 19.0$_{\pm0.4}$ & 20.1$_{\pm0.6}$ & 19.9$_{\pm1.4}$ & 13.6$_{\pm0.9}$ \\
\texttt{s}54 & 25.2$_{\pm1.5}$ & 23.2$_{\pm0.8}$ & 20.3$_{\pm0.5}$ & 20.4$_{\pm1.1}$ & 20.5$_{\pm1.3}$ & 15.8$_{\pm0.7}$ \\
\bottomrule
\end{tabular}}
\end{center}
\end{table}

\begin{table}[t]
\caption{Full success-rate grid for \textsc{ARISE} on AMC23: every anchor on the analyzed
segment, under the unperturbed 110-skill bank (step 2000) and all five perturbation axes.
Cells report mean $\pm$ 1 s.d.\ over evaluation repeats. The perturbed variants are fixed
across anchors; only the policy checkpoint varies. T5 is the no-skill arm.}
\label{tab:full-arise-amc23}
\begin{center}
\scalebox{0.95}{
\begin{tabular}{lcccccc}
\toprule
Anchor & Unperturbed & T1 & T2 & T3 & T4 & T5 \\
\midrule
\texttt{s}100  & 15.5$_{\pm0.8}$ & 16.0$_{\pm0.6}$ & 16.0$_{\pm1.1}$ & 16.8$_{\pm0.5}$ & 17.2$_{\pm1.0}$ & 17.2$_{\pm0.8}$ \\
\texttt{s}200  & 17.5$_{\pm0.5}$ & 18.6$_{\pm1.0}$ & 18.9$_{\pm0.4}$ & 16.9$_{\pm0.9}$ & 18.3$_{\pm0.5}$ & 19.4$_{\pm0.5}$ \\
\texttt{s}300  & 20.0$_{\pm0.9}$ & 20.0$_{\pm1.3}$ & 20.0$_{\pm0.5}$ & 20.0$_{\pm0.7}$ & 20.6$_{\pm1.1}$ & 20.9$_{\pm1.4}$ \\
\texttt{s}400  & 22.5$_{\pm1.0}$ & 23.6$_{\pm0.8}$ & 22.9$_{\pm1.5}$ & 23.9$_{\pm0.5}$ & 25.0$_{\pm1.3}$ & 23.9$_{\pm0.7}$ \\
\texttt{s}500  & 26.7$_{\pm0.6}$ & 23.3$_{\pm0.5}$ & 25.8$_{\pm0.7}$ & 27.1$_{\pm1.3}$ & 25.8$_{\pm0.6}$ & 27.1$_{\pm1.0}$ \\
\texttt{s}600  & 31.7$_{\pm1.1}$ & 30.0$_{\pm0.8}$ & 31.5$_{\pm1.0}$ & 32.5$_{\pm0.5}$ & 34.0$_{\pm0.5}$ & 32.5$_{\pm0.6}$ \\
\texttt{s}700  & 55.0$_{\pm1.2}$ & 53.3$_{\pm0.9}$ & 50.8$_{\pm0.8}$ & 54.2$_{\pm1.0}$ & 52.5$_{\pm0.9}$ & 52.5$_{\pm0.7}$ \\
\texttt{s}800  & 71.9$_{\pm1.3}$ & 68.8$_{\pm1.2}$ & 66.7$_{\pm0.7}$ & 72.9$_{\pm1.0}$ & 61.5$_{\pm1.0}$ & 67.7$_{\pm1.4}$ \\
\texttt{s}900  & 78.4$_{\pm1.2}$ & 75.0$_{\pm0.7}$ & 72.7$_{\pm1.5}$ & 79.5$_{\pm0.5}$ & 67.0$_{\pm0.9}$ & 73.9$_{\pm1.2}$ \\
\texttt{s}1000 & 87.4$_{\pm0.6}$ & 82.5$_{\pm0.9}$ & 81.5$_{\pm0.4}$ & 80.9$_{\pm1.1}$ & 76.5$_{\pm1.2}$ & 75.3$_{\pm1.0}$ \\
\bottomrule
\end{tabular}}
\end{center}
\end{table}

\paragraph{Dips inside the analyzed segment.}
Section~\ref{ssec:dose} retains every anchor up to the peak, dips included; this is
where they are named. Rating each dip by $|\Delta \SR|$ against the pooled standard
deviation of the two anchors it joins, \textsc{SkillRL}-ALFWorld contributes three
(\texttt{s0}$\to$\texttt{s5}, $-3.1$ points, $2.1\sigma$;
\texttt{s20}$\to$\texttt{s25}, $-8.1$, $3.0\sigma$;
\texttt{s40}$\to$\texttt{s45}, $-9.7$, $4.4\sigma$), \textsc{ReSkill}-ScienceWorld one
(\texttt{s180}$\to$\texttt{s240}, $-12.6$, $8.1\sigma$),
\textsc{ReSkill}-ALFWorld two (\texttt{s180}$\to$\texttt{s210}, $-21.0$,
$13.7\sigma$; and \texttt{s210}$\to$\texttt{s240}, $-13.1$, $17.2\sigma$, after which the
run recovers and sets its highest score at the final anchor), \textsc{SkillRL}-WebShop one
that is indistinguishable
from evaluation noise (\texttt{s40}$\to$\texttt{s50}, $-1.1$, $0.6\sigma$), and
\textsc{Evolving-RL}-Mind2Web two equally small reversals (\texttt{s9}$\to$\texttt{s14},
$-0.1$, $0.1\sigma$; \texttt{s19}$\to$\texttt{s24}, $-0.3$, $0.2\sigma$).
\textsc{ARISE}-AMC23 is monotone at the saved anchors. We retain every dip: the dose is the
amount of co-evolution a run has undergone, and a run that briefly loses ground has still
undergone it. No run in the six settings peaks and then stays down, so the clause of
Section~\ref{ssec:dose} that discards post-peak checkpoints never has to be applied.
Throughout, a standard deviation is the sample standard deviation over the individual
evaluation repeats.

\paragraph{The same grids as heatmaps.}
Figures~\ref{fig:heat-skillrl-alfworld}--\ref{fig:heat-arise-amc23} display the derived
quantities of Eqs.~\eqref{eq:drop}--\eqref{eq:rp} over the full anchor range. Each figure has
a $\Drop$ panel and an $\RP$ panel; the color scales are clipped so that the degradation
regime remains legible, because the earliest anchors reach values an order of magnitude
larger than the rest and would otherwise flatten the whole map into a single shade. Every
cell is derived from the per-anchor means in the corresponding full-grid table, and AVG is
the mean over the three scored content-level axes (T1--T3). The transition from blue on the
left to white and red on the right is the dependence forming: perturbations that help a
policy which has not yet co-evolved to the bank stop helping, and then start hurting.

\begin{figure}[t]
\begin{center}
\includegraphics[width=\textwidth]{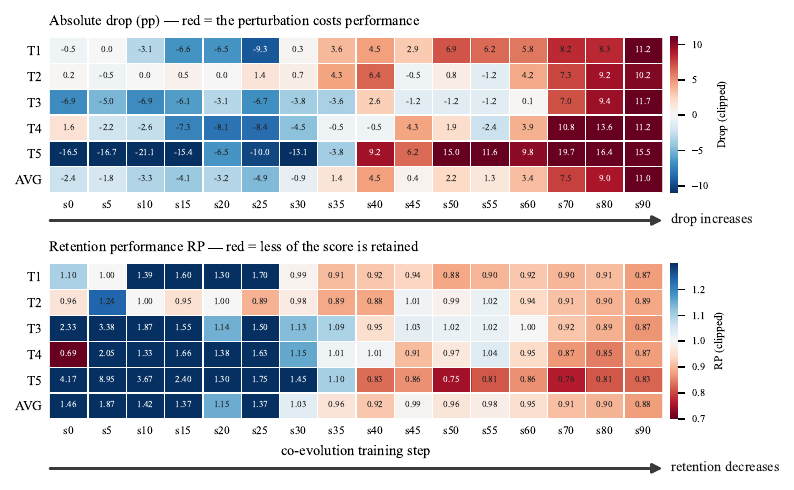}
\end{center}
\caption{Perturbation response across all checkpoints, \textsc{SkillRL} on ALFWorld.
Top: $\Drop$ in percentage points. Bottom: retention performance $\RP$. Color scales
clipped; cell values exact.}
\label{fig:heat-skillrl-alfworld}
\end{figure}

\begin{figure}[t]
\begin{center}
\includegraphics[width=\textwidth]{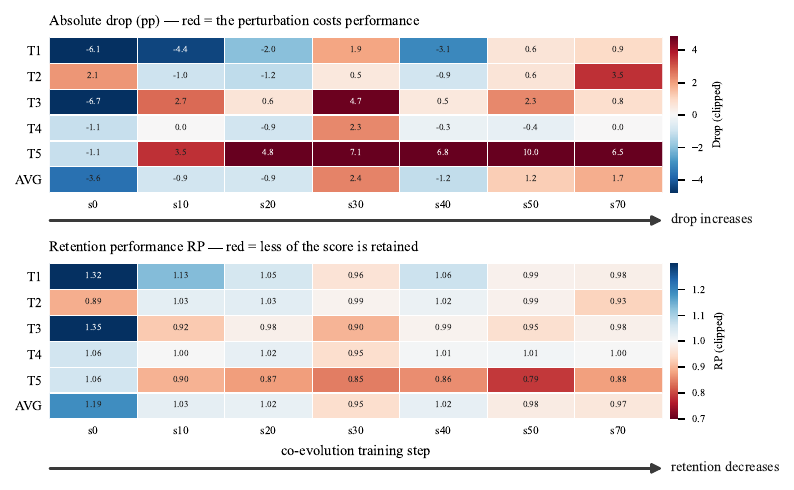}
\end{center}
\caption{Perturbation response across all checkpoints, \textsc{SkillRL} on WebShop.
Panels and conventions as in Figure~\ref{fig:heat-skillrl-alfworld}.}
\label{fig:heat-skillrl-webshop}
\end{figure}

\begin{figure}[t]
\begin{center}
\includegraphics[width=\textwidth]{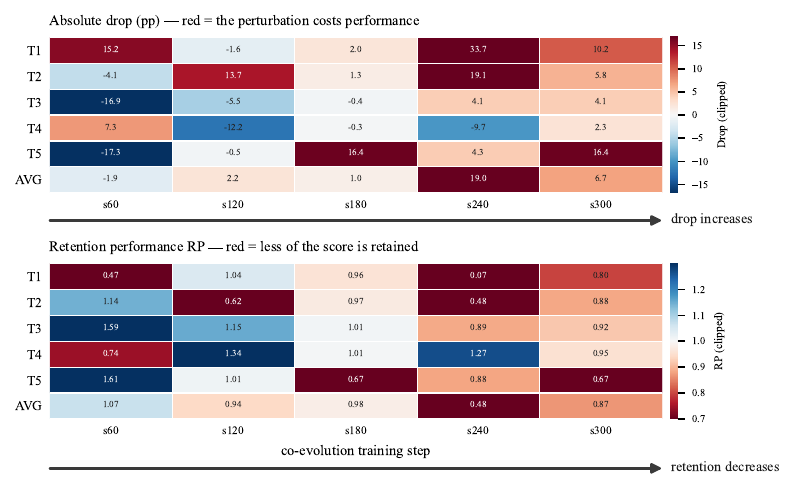}
\end{center}
\caption{Perturbation response across all checkpoints, \textsc{ReSkill} on
ScienceWorld. Panels and conventions as in Figure~\ref{fig:heat-skillrl-alfworld}.}
\label{fig:heat-reskill-scienceworld}
\end{figure}

\begin{figure}[t]
\begin{center}
\includegraphics[width=\textwidth]{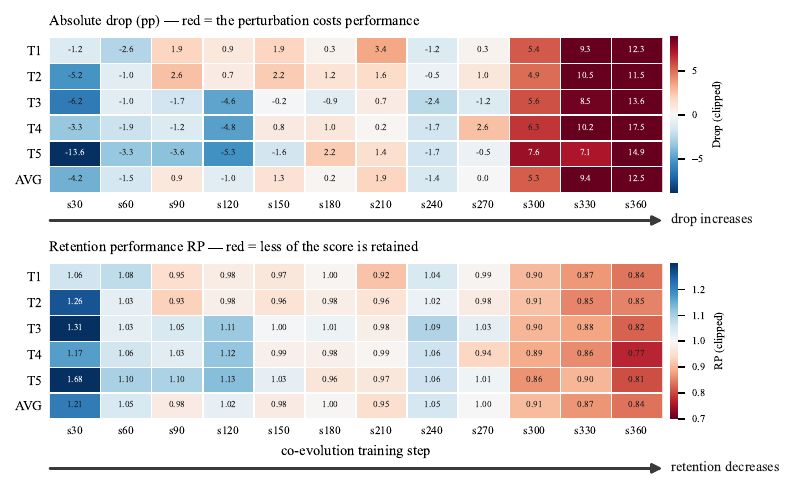}
\end{center}
\caption{Perturbation response across all checkpoints, \textsc{ReSkill} on ALFWorld.
Panels and conventions as in Figure~\ref{fig:heat-skillrl-alfworld}.}
\label{fig:heat-reskill-alfworld}
\end{figure}

\begin{figure}[t]
\begin{center}
\includegraphics[width=\textwidth]{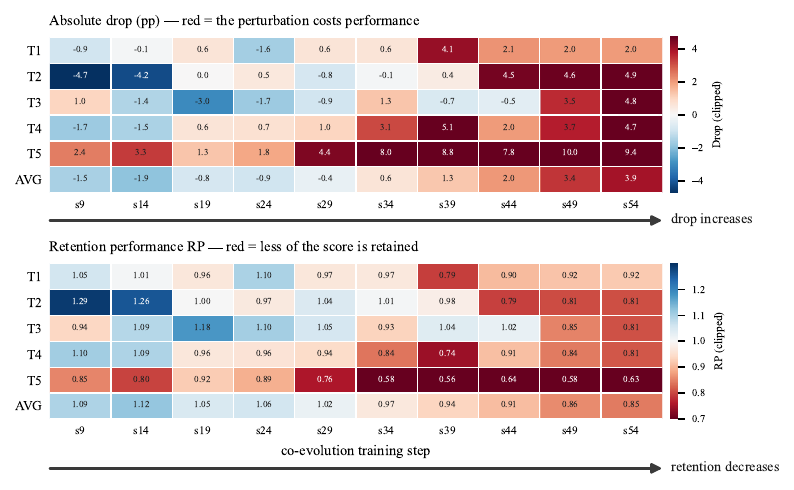}
\end{center}
\caption{Perturbation response across all checkpoints, \textsc{Evolving-RL} on Mind2Web.
Panels and conventions as in Figure~\ref{fig:heat-skillrl-alfworld}; cell values are derived
from the per-anchor means in Table~\ref{tab:full-evolvingrl-mind2web}.}
\label{fig:heat-evolvingrl-mind2web}
\end{figure}

\begin{figure}[t]
\begin{center}
\includegraphics[width=\textwidth]{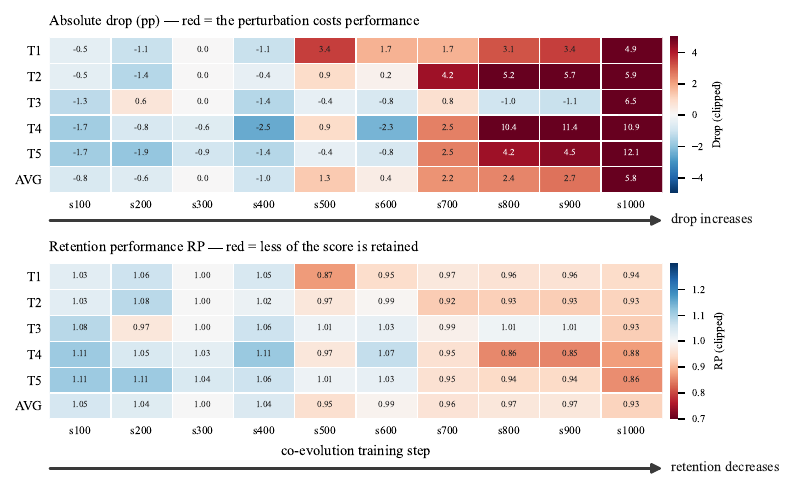}
\end{center}
\caption{Perturbation response across all checkpoints, \textsc{ARISE} on AMC23.
Panels and conventions as in Figure~\ref{fig:heat-skillrl-alfworld}; cell values are derived
from the per-anchor means in Table~\ref{tab:full-arise-amc23}.}
\label{fig:heat-arise-amc23}
\end{figure}

\subsection{Co-evolution Fragility Curves, All Six Settings}
\label{app:dose-all}

Figure~\ref{fig:dose-reskill} in the main text reports \textsc{ReSkill} on ALFWorld.
The other five settings follow in the same layout: \textsc{SkillRL} on WebShop
(Figure~\ref{fig:dose-skillrl}) and ALFWorld
(Figure~\ref{fig:dose-skillrl-alfworld}), \textsc{ReSkill} on ScienceWorld
(Figure~\ref{fig:dose-reskill-scienceworld}), \textsc{Evolving-RL} on Mind2Web
(Figure~\ref{fig:dose-evolvingrl}) and \textsc{ARISE} on AMC23
(Figure~\ref{fig:dose-arise}).

\begin{figure}[t]
\begin{center}
\includegraphics[width=\textwidth]{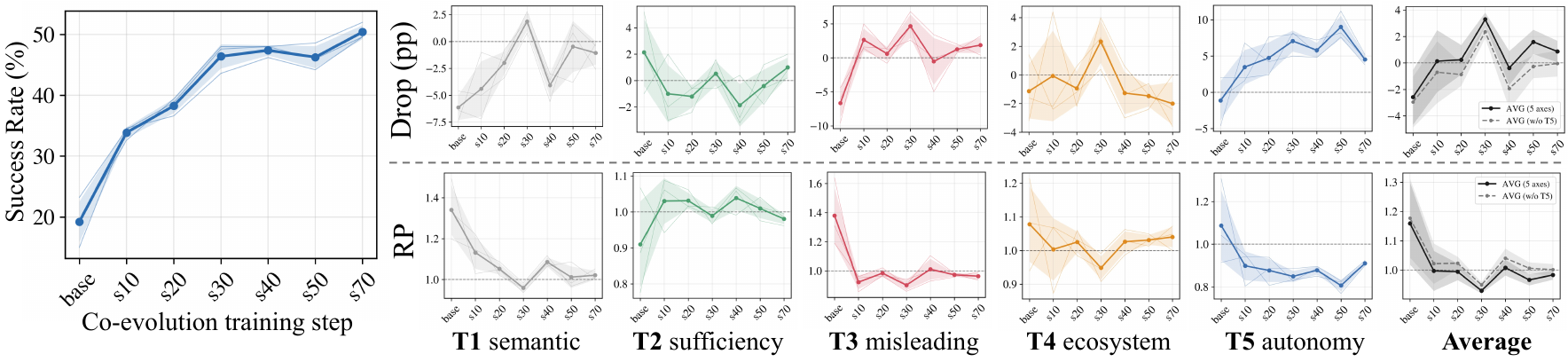}
\end{center}
\caption{Co-evolution fragility curve for \textsc{SkillRL} on WebShop; same layout as
Figure~\ref{fig:dose-reskill}. This is the sole setting in which no content-level drop slope
reaches significance.}
\label{fig:dose-skillrl}
\end{figure}

\begin{figure}[t]
\begin{center}
\includegraphics[width=\textwidth]{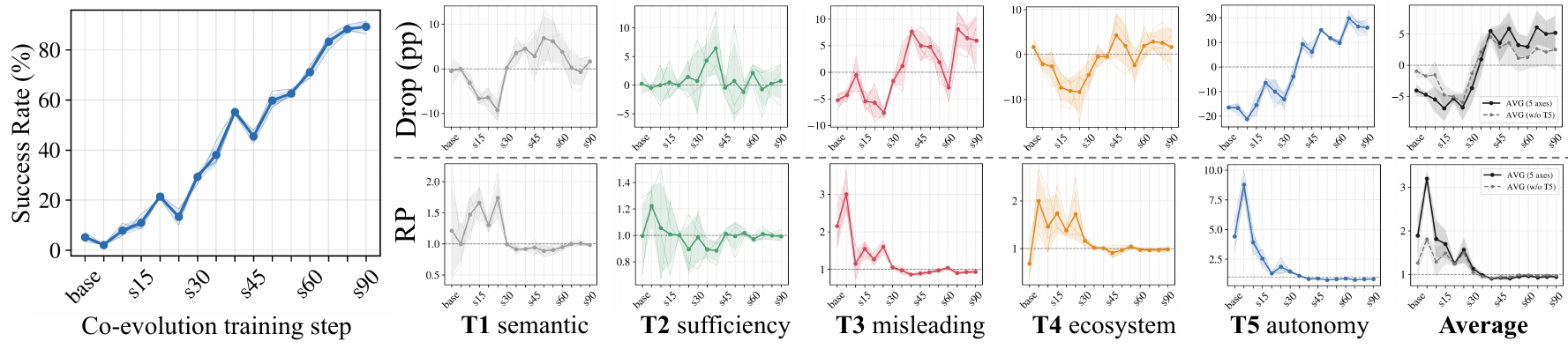}
\end{center}
\caption{Co-evolution fragility curve for \textsc{SkillRL} on ALFWorld; same layout as
Figure~\ref{fig:dose-skillrl}. This setting has the widest competence range of the six,
from $5.2\%$ to $88.8\%$ over sixteen anchors, so the transition is easiest to read
here. Negative drops and $\RP > 1$ at the earliest anchors mean the perturbation
\emph{improved} performance, which is admissible for a policy that has not yet
co-evolved to the bank (Proposition~\ref{prop:peak}, third clause) rather than a
violation of the one-sided prediction. Error bars span $\pm 1$ s.d.\ over three
evaluation repeats that share the policy checkpoint and the same fixed set of rewritten
banks and differ only in the rollout sampling seed (temperature 0.4); because each
setting comes from a single co-evolution run, these bands quantify evaluation noise and
say nothing about variability across training seeds.}
\label{fig:dose-skillrl-alfworld}
\end{figure}

\begin{figure}[t]
\begin{center}
\includegraphics[width=\textwidth]{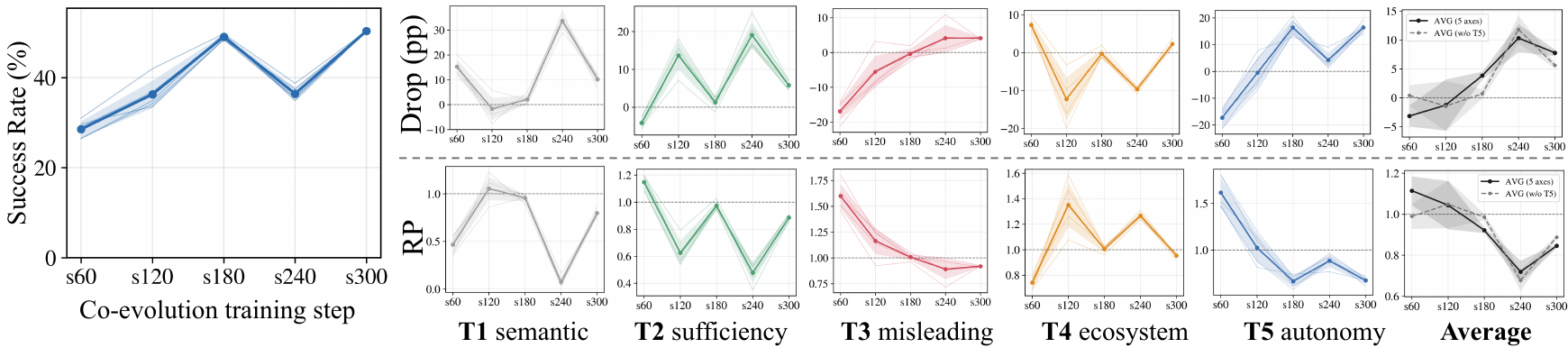}
\end{center}
\caption{Co-evolution fragility curve for \textsc{ReSkill} on ScienceWorld; same layout
as Figure~\ref{fig:dose-reskill}. Error bars span five evaluation repeats, and this setting
has the fewest anchors of the six (five), yielding the widest slope intervals.}
\label{fig:dose-reskill-scienceworld}
\end{figure}

\begin{figure}[t]
\begin{center}
\includegraphics[width=\textwidth]{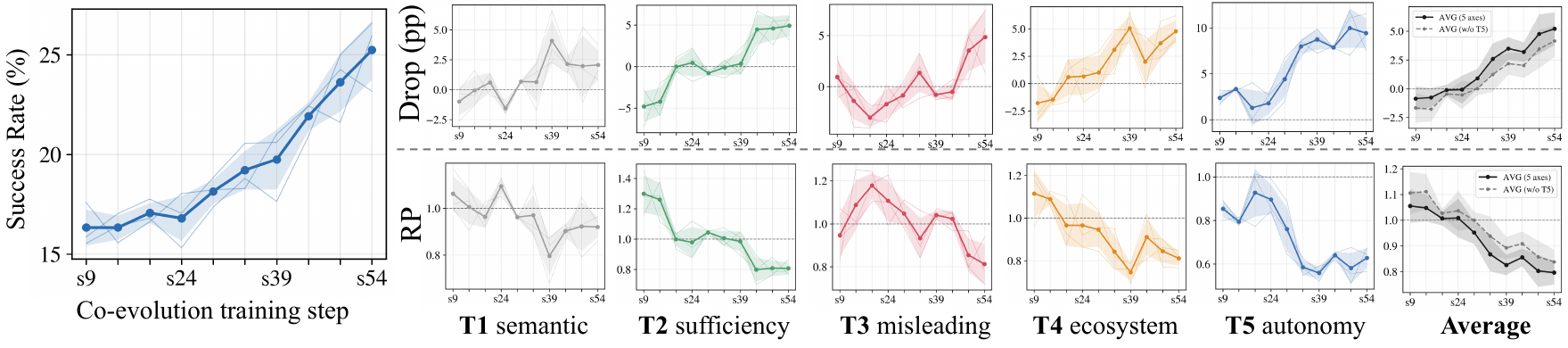}
\end{center}
\caption{Co-evolution fragility curve for \textsc{Evolving-RL} on Mind2Web; same layout
as Figure~\ref{fig:dose-reskill}.}
\label{fig:dose-evolvingrl}
\end{figure}

\begin{figure}[t]
\begin{center}
\includegraphics[width=\textwidth]{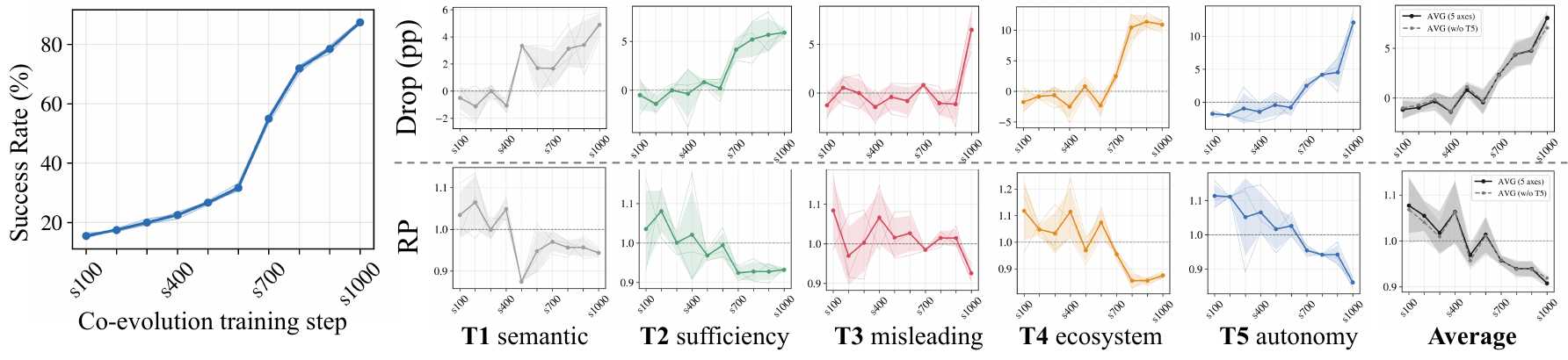}
\end{center}
\caption{Co-evolution fragility curve for \textsc{ARISE} on AMC23; same layout
as Figure~\ref{fig:dose-reskill}.}
\label{fig:dose-arise}
\end{figure}

\subsection{Drop and Retention Are One Measurement on Two Scales}
\label{app:dropvsrp}

The two quantities plotted in every dose figure are algebraically linked. Substituting the
definition of $\Drop_a$ into that of $\RP_a$,
\begin{equation}
\RP_a(t) \;=\; \frac{\SR(\Bankstar_a)}{\SR(\Bankstar)}
\;=\; \frac{\SR(\Bankstar) - \Drop_a(t)}{\SR(\Bankstar)}
\;=\; 1 - \frac{\Drop_a(t)}{\SR(\Bankstar)},
\label{eq:dropvsrp}
\end{equation}
so retention is an affine transform of the drop: a flip about the horizontal axis, rescaled
by $1/\SR(\Bankstar)$. Whenever the unperturbed rate is close to constant across anchors,
the retention row is therefore an almost exact mirror image of the drop row, and the sole
source of divergence between the two is variation in $\SR(\Bankstar)$ itself.

\paragraph{A sign reversal along training.}
On \textsc{SkillRL}-ALFWorld, $\RP_{\text{T3}}=2.33$ at \texttt{s0} and $3.38$ at
\texttt{s5}: a misleading bank more than doubles the score of a policy that has not yet
co-evolved with it. By \texttt{s90}, $\RP_{\text{T3}}$ falls to $0.86$, while T5 moves
from $4.17$ to $0.82$. The same bank that initially obstructs the policy thus becomes a
source of competence and, eventually, dependence.

Across the six repeatedly evaluated settings,
Figure~\ref{fig:dose-reskill-scienceworld} spans a relatively narrow competence range
($28.5\%$ to $50.3\%$), so its drop and retention rows mirror each other almost perfectly.
Figure~\ref{fig:dose-skillrl-alfworld} spans $5.2\%$ to
$88.8\%$; there the small early denominators stretch retention onto a different scale
altogether---T5 exceeds $\RP = 7$ at the base anchor---and the two rows take visibly
different shapes.

We report both regardless, for opposite reasons: the absolute drop is what a deployment
actually loses, while retention is the only one of the two that remains comparable across
checkpoints---or across settings---whose ability differs by an order of magnitude. At a fixed
anchor neither carries information the other lacks; what retention adds is comparability
along the dose axis.

\subsection{Dependence Mediates Fragility}
\label{app:mediation}

The dose curves plot fragility against the co-evolution step, but the step is only a
proxy: what should actually govern how much a fixed corruption costs is how much the
policy has come to \emph{rely} on the bank. Figure~\ref{fig:dependence-fragility} tests
this directly. For every checkpoint we place one point at $(x,y)$ with
$x$ the dependence---the drop in $\SR$ when the bank is removed entirely (the no-skill /
autonomy contrast)---and $y$ the fragility---the mean $\Drop$ over the five corruption
tiers. Reading fragility against measured dependence removes the proxy, which matters
because on \textsc{SkillRL}-WebShop the proxy fails outright: dependence there is not
monotone in the step (it reverses across an \texttt{s80}--\texttt{s110} valley), so
dependence and step are essentially uncorrelated, yet the dependence--fragility relation
still holds.

The two quantities lie on a single increasing line. Pooling the $32$ checkpoints across
both settings gives Spearman $\rho=+0.82$ ($95\%$ CI $[+0.68,+0.89]$, $p=7.4\times
10^{-9}$) and an ordinary-least-squares slope of $+0.20$ pp of fragility per pp of
dependence ($95\%$ bootstrap CI $[+0.15,+0.28]$, excluding zero); each setting is
individually significant (\textsc{SkillRL}-ALFWorld $\rho=+0.78$, $p=1.5\times10^{-4}$,
$n=18$; \textsc{SkillRL}-WebShop $\rho=+0.79$, $p=7.9\times10^{-4}$, $n=14$). The relation
is not carried by a few extreme points: the shaded band is a $4000$-sample bootstrap over
anchors and stays narrow throughout. Crucially the line passes through the third quadrant.
At the earliest checkpoints dependence is negative---removing the bank \emph{raises}
$\SR$, because the policy has not yet co-evolved to it---and there fragility is negative
too: corrupting a bank the policy does not rely on \emph{helps}. The same axis that
predicts vulnerability at high dependence predicts a benefit at negative dependence, which
is exactly the sign structure of Proposition~\ref{prop:peak}. Dependence, not elapsed
training, is what fragility tracks; the dose effect of Section~\ref{ssec:dose} is the
composition of two links, training builds dependence and dependence creates fragility.

\begin{figure}[t]
\begin{center}
\includegraphics[width=0.82\textwidth]{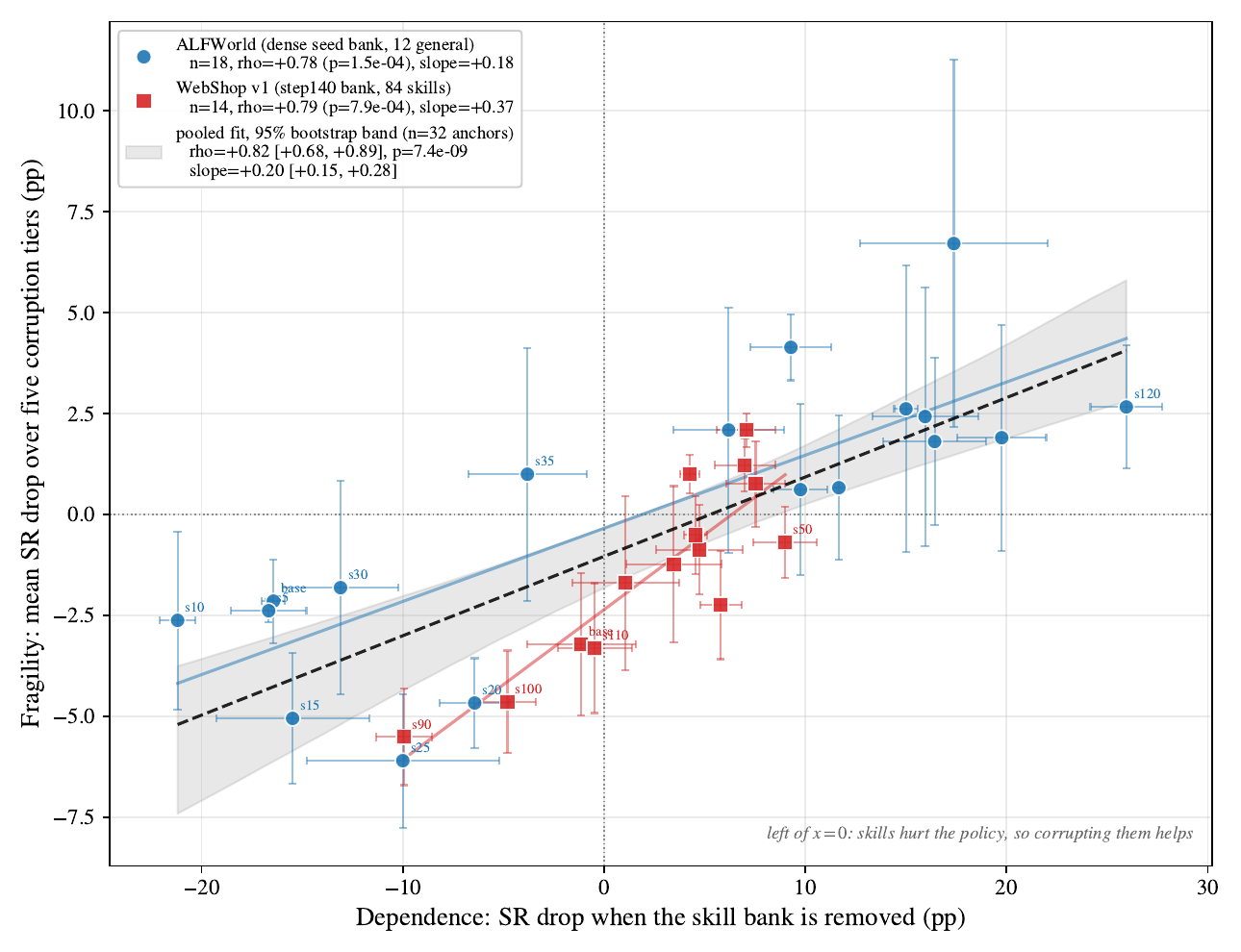}
\end{center}
\caption{Fragility scales with acquired dependence. Each point is one co-evolution
checkpoint (\textsc{SkillRL}-ALFWorld, circles; \textsc{SkillRL}-WebShop, squares);
error bars are $\pm 1$ s.d.\ of the same contrast over the three evaluation repeats.
$x$: dependence, the $\SR$ drop when the skill bank is removed. $y$: fragility, the mean
$\SR$ drop over the five corruption tiers. Solid lines are per-setting fits; the dashed
line and grey band are the pooled fit with its $95\%$ bootstrap envelope over anchors.
Points left of $x=0$ are checkpoints that have not yet co-evolved to the bank, where
removing or corrupting skills improves performance.}
\label{fig:dependence-fragility}
\end{figure}

\subsection{Surface-Form Dependence and Prompt Share}
\label{app:kappa}

\paragraph{How much of the drop is about wording?}
Comparing the two ends of the axis ordering gives a cheap diagnostic of surface-form
dependence,
\begin{equation}
\kappa \;=\; \frac{\SR(\Bankstar) - \SR(\Bankstar_{\text{T1}})}{\SR(\Bankstar) - \SR(\Bankstar_{\text{T3}})},
\label{eq:kappa}
\end{equation}
the damage done by rewriting the words relative to the damage done by rewriting the
strategy. A policy that reads a skill for its meaning should have $\kappa$ near
zero, since a synonymous rewrite costs it nothing while a wrong strategy costs it a
great deal; a policy using the bank as a look-up table has $\kappa$ near one or
above. Measured values are $\kappa = 0.78$ on \textsc{SkillRL}-ALFWorld and $2.14$ on
\textsc{ReSkill}-ScienceWorld; the latter exceeds one, meaning a meaning-preserving
change hurts more than a wrong strategy does. On the other two \textsc{SkillRL}/\textsc{ReSkill}
settings, the T3 axis does not lower the success rate below the peak, so $\kappa$ is undefined (negative
denominator). We report $\kappa$ as a diagnostic and deliberately keep it out of
Table~\ref{tab:metrics}: the metric set is closed, and a ratio of two differences is
too unstable near a small denominator to carry a pass line.

\paragraph{An observation, not a prediction.}
Injected skill text occupies an estimated 15\% to 40\% of the prompt across the four
\textsc{SkillRL}/\textsc{ReSkill} settings included in this diagnostic, and those with the
larger share are \emph{not} systematically the
ones with the higher worst-case rates. We record this because it is the first thing a
reader will ask about and the measurement is nearly free, but it forms no part of the
theory of Section~\ref{sec:theory} and we make no claim of a relationship.

\subsection{T1 Meaning-Preservation Quality Assurance}
\label{app:t1quality}

The T1 axis is defined as a meaning-preserving rewrite: the skill's wording changes but its
procedural content should not. Because the rewriting suite uses a language model rather than
a deterministic operator, we verify this property with an independent, blinded audit whose
design specifically defends against the two failure modes a reviewer would flag: the judge
being biased toward the rewriter (\emph{self-endorsement}) and the judge being insensitive
(a \emph{rubber-stamp}).

\paragraph{Protocol.}
We audit \emph{all} 190 T1 rewrites in the four \textsc{SkillRL}/\textsc{ReSkill} settings
(91 on \textsc{SkillRL}-ALFWorld,
84 on \textsc{SkillRL}-WebShop, 7 on \textsc{ReSkill}-ScienceWorld, 8 on
\textsc{ReSkill}-ALFWorld), so there is no sampling. Every pair is scored on a four-level
rubric---\emph{equivalent}, \emph{minor drift}, \emph{meaning altered},
\emph{reversed/unrelated}---by two independent judges from different providers:
\texttt{qwen3-max} (Alibaba) as the primary judge, cross-family with the rewriter (GPT-4o),
and \texttt{gpt-4o-2024-08-06} (OpenAI) as the secondary judge. The audit is blinded: pair
order is shuffled, the original is randomly placed on the A or B side, and axis labels are
kept only in an answer key. A human annotator re-checks 50\% of the T1 pairs
under the same blind protocol.

\paragraph{Judge validity: positive and negative controls.}
To confirm that the judges are neither over-eager nor over-lenient, the same blinded batch
includes 20 identity pairs (a skill against itself) and 40 negative controls (a skill
against its T2 or T3 rewrite, which by construction alter the prescribed action). Both
judges rate all 20/20 identity pairs \emph{equivalent} and all 40/40 negative controls as
non-equivalent (mean rating $\le 0.36$ vs. $1.00$ on identity), so they are not rubber
stamps and not oversensitive.

\paragraph{Results.}
Under this test both judges rate \emph{100\%} of the 190 T1 pairs as meaning-equivalent
($0/190$ altered; 95\% Wilson CI $[0.0\%, 1.9\%]$), with perfect binary agreement between
judges (Cohen's $\kappa = 1.0$) and a four-level agreement of 95.8\% (all disagreements fall
within the negative controls, over the \emph{degree} of alteration, never over its
direction). The human re-check on 95 of the 190 T1 pairs (50\%) concurs, with no meaning
changes found. In these four settings, T1 rewrites therefore preserve meaning, and the
corresponding consistency drop in Section~\ref{ssec:reportcard} measures sensitivity to
wording rather than rewrite noise.

\paragraph{Skill-bank statistics.}
Table~\ref{tab:bank-stats} summarizes the scale and growth mechanism of the skill banks for
the four \textsc{SkillRL}/\textsc{ReSkill} settings. \textsc{SkillRL} starts from a fixed set of human-authored seed
skills and accumulates new skills during training (monotone growth); \textsc{ReSkill} builds
from scratch under a fixed library cap.

\begin{table}[h]
\caption{Skill-bank configuration and size for the four \textsc{SkillRL}/\textsc{ReSkill}
settings.}
\label{tab:bank-stats}
\begin{center}
\begin{tabular}{lllcc}
\toprule
Method & Environment & Mechanism & Initial & Delivered \\
\midrule
\textsc{SkillRL} & ALFWorld     & seed + accumulate            & 55 & 91 \\
\textsc{SkillRL} & WebShop      & seed + accumulate            & 66 & 84 \\
\textsc{ReSkill} & ScienceWorld & from scratch, capped         &  0 &  7 \\
\textsc{ReSkill} & ALFWorld     & from scratch, capped         &  0 &  8 \\
\bottomrule
\end{tabular}
\end{center}
\end{table}

\subsection{The Cost in Standalone Ability}
\label{app:noskill}

If co-evolution builds dependence, the natural question is what that dependence is bought
with. Figure~\ref{fig:noskill-baseline} trains a pure-RL policy on
\textsc{SkillRL}-ALFWorld with the skill mechanism disabled throughout (arm A, no bank at
any step) and reads it against the co-evolved delivered checkpoint of the same setting,
both with the bank present ($84.1\%$, $n=5$) and with it removed ($57.6\%$, $n=5$;
step-$115$ bank, evaluated on the \texttt{eval\_in\_distribution} split at temperature
$0.4$). Two things stand out. First, the co-evolved policy stripped of its bank scores
$57.6\%$, whereas a policy that never trained with skills reaches a last-five-checkpoint
mean of $70.6\%$ and a peak of $87.5\%$: co-evolution does not merely add a removable
crutch, it leaves the underlying policy \emph{worse at acting alone} than one that was
never given the crutch. The dependence measured in
Appendix~\ref{app:mediation} is therefore not free---it is standalone ability traded
away. Second, the pure-RL run is itself unstable: it peaks at $87.5\%$ near step $125$
and then falls to $53.1\%$ by step $150$, with entropy and KL rising together over the
same window (panels B, C). Neither training regime yields a monotone capability curve;
the distinctive cost of co-evolution is the additional erosion of the skill-free policy.

The comparison is between two evaluations that share the \texttt{eval\_in\_distribution}
split and the temperature-$0.4$ sampling but differ in sample count (the training-time
validation draws $64$ tasks, the \CREPE{} reference $140$) and in the ALFWorld backend
(textual \texttt{AlfredTWEnv} during training versus the simulator-backed
\texttt{AlfredThorEnv} at evaluation); the reference lines should be read as the level of
the co-evolved policy on the shared split, not as a point comparison to two decimal
places. We report this as a supporting observation rather than a headline result.

\begin{figure}[t]
\begin{center}
\includegraphics[width=\textwidth]{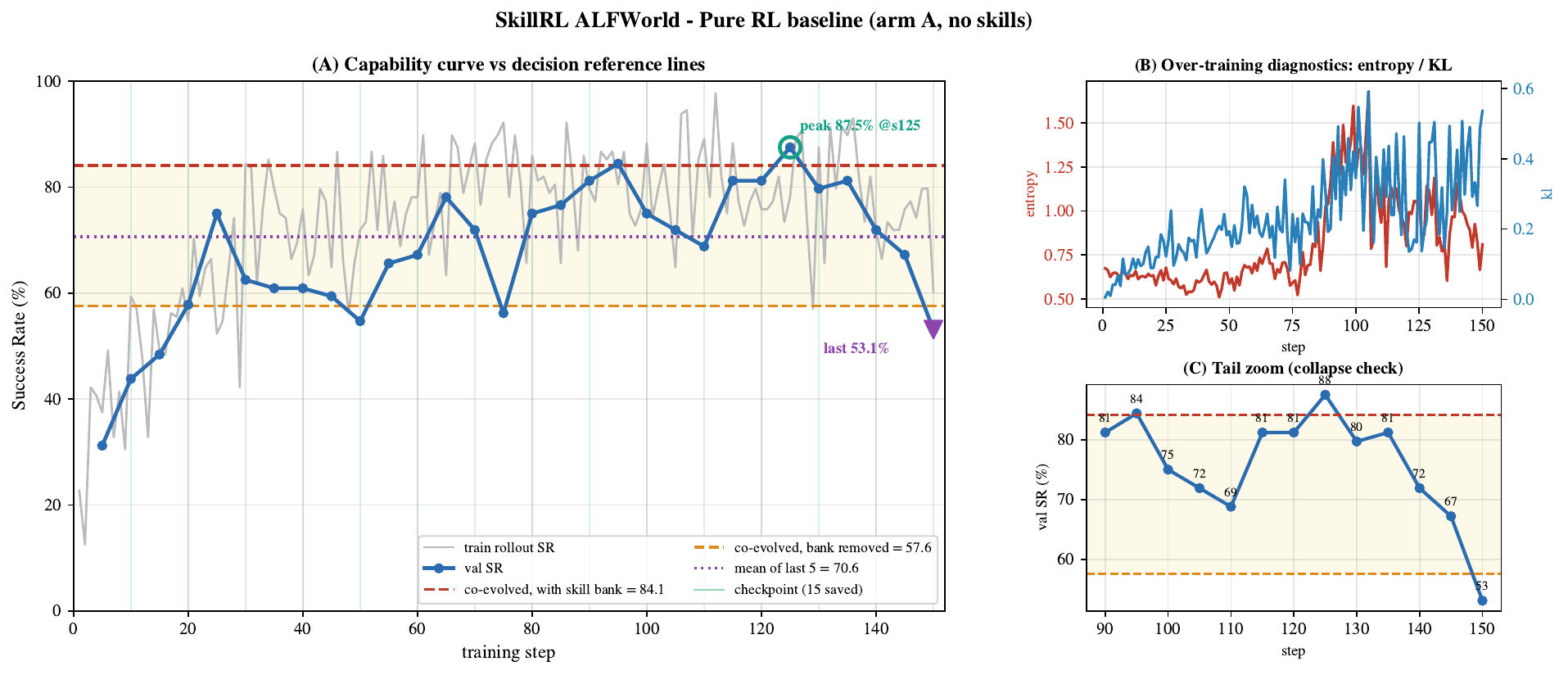}
\end{center}
\caption{Pure-RL baseline (arm A, skills disabled) on \textsc{SkillRL}-ALFWorld.
(A) Validation $\SR$ over training, against the co-evolved delivered checkpoint with the
bank present (red, $84.1\%$) and removed (orange, $57.6\%$); the co-evolved policy with
its bank removed sits below even the pure-RL last-five mean (purple dotted, $70.6\%$).
(B) Entropy and KL rise together in the second half, the signature of an over-training
collapse. (C) Tail zoom: the run falls from its step-$125$ peak to $53.1\%$ at step $150$.
Both curves and reference lines use the \texttt{eval\_in\_distribution} split at
temperature $0.4$; see the text for the difference in sample count and backend.}
\label{fig:noskill-baseline}
\end{figure}

\end{document}